\documentclass{article}

\PassOptionsToPackage{numbers}{natbib}
\usepackage[final]{neurips_2020}
\usepackage[utf8]{inputenc} 
\usepackage[T1]{fontenc}    
\usepackage[colorlinks=true, linkcolor=blue, citecolor=blue,urlcolor=black]{hyperref}       
\usepackage{url}            
\usepackage{booktabs}       
\usepackage{amsfonts}       
\usepackage{nicefrac}       
\usepackage{microtype}      

\usepackage{amsthm}
\usepackage{algorithm}
\usepackage{algorithmicx}
\usepackage[noend]{algcompatible}
\usepackage{bbm}      
\usepackage{mathtools}    
\usepackage{amsopn}
\usepackage{amssymb}
\usepackage{graphicx}
\usepackage{xcolor}

\definecolor{Green}{rgb}{0.13, 0.65, 0.3}
\definecolor{Amber}{rgb}{0.3, 0.5, 1.0}
\usepackage[]{color-edits}
\addauthor{HL}{Green}
\addauthor{TJ}{Amber}

\newcommand{\inner}[1]{ \left\langle {#1} \right\rangle }
\newcommand{\Ind}[1]{ \field{I}{\{{#1}\}} }

\newcommand{\norm}[1]{\left\|{#1}\right\|}
\newcommand{\gap}{\Delta}
\newcommand{\hatl}{\widehat{\ell}}
\newcommand{\hatL}{\widehat{L}}

\newtheorem{theorem}{Theorem}
\newtheorem{lemma}[theorem]{Lemma}
\newtheorem{corollary}[theorem]{Corollary}

\allowdisplaybreaks

\usepackage{amsmath}
\DeclareMathOperator*{\argmin}{\arg\!\min}

\usepackage{amsthm}
\usepackage{amsmath}
\usepackage{amssymb}
\usepackage{graphicx}
\usepackage{mathtools}
\usepackage{enumerate}
\usepackage{enumitem}
\usepackage{footnote}
\usepackage{float}
\usepackage{xspace}
\usepackage{multirow}
\usepackage{nicefrac}
\usepackage{wrapfig}
\usepackage{framed}
\usepackage{url}
\usepackage{tikz}
\usetikzlibrary{arrows,chains,matrix,positioning,scopes}

\newcommand{\calM}{{\mathcal{M}}}

\newcommand{\paralog}{\beta}
\newcommand{\cnt}{B}
\newcommand{\gapmin}{\Delta_{\textsc{min}}}

\newcommand{\field}[1]{\mathbb{#1}}

\newcommand{\fR}{\field{R}}

\newcommand{\E}{\field{E}}

\newcommand{\Reg}{{\text{\rm Reg}}}

\newcommand{\order}{\ensuremath{\mathcal{O}}}
\newcommand{\otil}{\ensuremath{\widetilde{\mathcal{O}}}}

\newcommand{\Holder}{{H{\"o}lder}\xspace}

\newcommand{\opt}{\mathring{q}}
\newcommand{\optpi}{\mathring{\pi}}

\newcommand{\rbr}[1]{\left(#1\right)}
\newcommand{\sbr}[1]{\left[#1\right]}
\newcommand{\cbr}[1]{\left\{#1\right\}}

\DeclareFontFamily{OMX}{MnSymbolE}{}
\DeclareFontShape{OMX}{MnSymbolE}{m}{n}{
    <-6>  MnSymbolE5
   <6-7>  MnSymbolE6
   <7-8>  MnSymbolE7
   <8-9>  MnSymbolE8
   <9-10> MnSymbolE9
  <10-12> MnSymbolE10
  <12->   MnSymbolE12}{}
\DeclareSymbolFont{mnlargesymbols}{OMX}{MnSymbolE}{m}{n}
\SetSymbolFont{mnlargesymbols}{bold}{OMX}{MnSymbolE}{b}{n}
\DeclareMathDelimiter{\llangle}{\mathopen}{mnlargesymbols}{'164}{mnlargesymbols}{'164}
\DeclareMathDelimiter{\rrangle}{\mathclose}{mnlargesymbols}{'171}{mnlargesymbols}{'171}

\usepackage{prettyref}
\newcommand{\pref}[1]{\prettyref{#1}}

\newcommand{\savehyperref}[2]{\texorpdfstring{\hyperref[#1]{#2}}{#2}}
\newrefformat{eq}{\savehyperref{#1}{Eq.~\eqref{#1}}}
\newrefformat{eqn}{\savehyperref{#1}{Equation~\eqref{#1}}}
\newrefformat{lem}{\savehyperref{#1}{Lemma~\ref*{#1}}}
\newrefformat{def}{\savehyperref{#1}{Definition~\ref*{#1}}}
\newrefformat{line}{\savehyperref{#1}{line~\ref*{#1}}}
\newrefformat{thm}{\savehyperref{#1}{Theorem~\ref*{#1}}}
\newrefformat{corr}{\savehyperref{#1}{Corollary~\ref*{#1}}}
\newrefformat{cor}{\savehyperref{#1}{Corollary~\ref*{#1}}}
\newrefformat{col}{\savehyperref{#1}{Corollary~\ref*{#1}}}
\newrefformat{sec}{\savehyperref{#1}{Section~\ref*{#1}}}
\newrefformat{app}{\savehyperref{#1}{Appendix~\ref*{#1}}}
\newrefformat{assum}{\savehyperref{#1}{Assumption~\ref*{#1}}}
\newrefformat{ex}{\savehyperref{#1}{Example~\ref*{#1}}}
\newrefformat{fig}{\savehyperref{#1}{Figure~\ref*{#1}}}
\newrefformat{alg}{\savehyperref{#1}{Algorithm~\ref*{#1}}}
\newrefformat{rem}{\savehyperref{#1}{Remark~\ref*{#1}}}
\newrefformat{conj}{\savehyperref{#1}{Conjecture~\ref*{#1}}}
\newrefformat{prop}{\savehyperref{#1}{Proposition~\ref*{#1}}}
\newrefformat{proto}{\savehyperref{#1}{Protocol~\ref*{#1}}}
\newrefformat{prob}{\savehyperref{#1}{Problem~\ref*{#1}}}
\newrefformat{claim}{\savehyperref{#1}{Claim~\ref*{#1}}}
\newrefformat{que}{\savehyperref{#1}{Question~\ref*{#1}}}

\newcommand{\ftrl}{\textsc{FTRL}\xspace}

\title{Simultaneously Learning Stochastic and Adversarial Episodic MDPs with Known Transition}

\author{%
  Tiancheng Jin \\
  University of Southern California \\
  \texttt{tiancheng.jin@usc.edu} \\
  \And
  Haipeng Luo \\
  University of Southern California \\
  \texttt{haipengl@usc.edu} \\
}

\begin{document}

\maketitle

\begin{abstract}
 
This work studies the problem of learning episodic Markov Decision Processes with known transition and bandit feedback.
We develop the first algorithm with a ``best-of-both-worlds'' guarantee: it achieves $\order(\log T)$ regret when the losses are stochastic, and simultaneously enjoys worst-case robustness with $\otil(\sqrt{T})$ regret even when the losses are adversarial, where $T$ is the number of episodes.
More generally, it achieves $\otil(\sqrt{C})$ regret in an intermediate setting where the losses are corrupted by a total amount of $C$.
Our algorithm is based on the Follow-the-Regularized-Leader method from~\citet{zimin2013}, with a novel hybrid regularizer inspired by recent works of Zimmert et al.~\citep{zimmert2019optimal,zimmert2019beating} for the special case of multi-armed bandits.
Crucially, our regularizer admits a non-diagonal Hessian with a highly complicated inverse.
Analyzing such a regularizer and deriving a particular self-bounding regret guarantee is our key technical contribution and might be of independent interest.
 
\end{abstract}

\section{Introduction} 
We study the problem of learning episodic Markov Decision Processes (MDPs).
In this problem, a learner interacts with the environment through $T$ episodes.
In each episode, the learner starts from a fixed state, then sequentially selects one of the available actions and transits to the next state according to a fixed transition function for a fixed number of steps.
The learner observes only the visited states and the loss for each visited state-action pair,
and her goal is to minimize her regret, the difference between her total loss over $T$ episodes and that of the optimal fixed policy in hindsight. 

When the losses are adversarial and can change arbitrarily between episodes, the state-of-the-art is achieved by the UOB-REPS algorithm of~\citep{jin2019learning} with near-optimal regret $\otil(\sqrt{T})$ (ignoring dependence on other parameters).
On the other hand, the majority of the literature focuses on the stochastic/i.i.d. loss setting where the loss for each state-action pair follows a fixed distribution. 
For example, \citet{azar2017minimax} achieve the minimax regret $\otil(\sqrt{T})$ in this case.
Moreover, the recent work of~\citet{simc2019} shows the first non-asymptotic gap-dependent regret bound of order $\order(\log T)$ for this problem, which is considerably more favorable than the worst-case $\otil(\sqrt{T})$ regret.

A natural question then arises: {\it is it possible to achieve the best of both worlds with one single algorithm?}
In other words, can we achieve $\order(\log T)$ regret when the losses are stochastic, and {\it simultaneously} enjoy worst-case robustness with $\otil(\sqrt{T})$ regret when the losses are adversarial?
Considering that the existing algorithms from~\citep{jin2019learning} and~\citep{simc2019} for these two settings are drastically different, it is highly unclear whether this is possible.

In this work, we answer the question affirmatively and develop the first algorithm with such a best-of-both-worlds guarantee, under the condition that the transition function is known. 
We emphasize that even in the case with known transition, the problem is still highly challenging.
For example, the adversarial case was studied in~\citep{zimin2013} and still requires using the Follow-the-Regularized-Leader (\ftrl) or Online Mirror Descent framework from the online learning literature (see e.g.,~\citep{Hazan16}) over the {\it occupancy measure} space, which the UOB-REPS algorithm~\citep{jin2019learning} adopts as well. 
This is still significantly different from the algorithms designed for the stochastic setting.

Moreover, our algorithm achieves the logarithmic regret $\order(\log T)$ for a much broader range of situations besides the usual stochastic setting.
In fact, neither independence nor identical distributions are required, as long as a certain gap condition similar to that of~\citep{zimmert2019optimal} (for multi-armed bandits) holds (see \pref{eq:self_bounding_constraint}).
Even more generally, our algorithm achieves $\otil(\log T + \sqrt{C})$ regret in an intermediate setting where the losses are corrupted by a total amount of $C$. 
This bound smoothly interpolates between the logarithmic regret for the stochastic setting and the worst-case $\otil(\sqrt{T})$ regret  for the adversarial setting as $C$ increases from $0$ to $T$.

\paragraph{Techniques.}
Our algorithm is mainly inspired by recent advances in achieving best-of-both-worlds guarantees for the special case of multi-armed bandits or semi-bandits~\citep{wei2018more, zimmert2019optimal, zimmert2019beating}.
These works show that, perhaps surprisingly, such guarantees can be obtained with the standard \ftrl framework originally designed only for the adversarial case.
All we need is a carefully designed regularizer and a particular analysis that relies on proving a certain kind of {\it self-bounding regret bounds}, which then automatically implies $\order(\log T)$ regret for the stochastic setting, and more generally $\otil(\sqrt{C})$ regret for the case with $C$ corruption.

We greatly extend this idea to the case of learning episodic MDPs.
As mentioned, \citet{zimin2013} already solved the adversarial case using \ftrl over the occupancy measure space, in particular with Shannon entropy as the regularizer in the form $\sum_{s,a}q(s,a) \ln q(s,a)$, where $q(s,a)$ is the occupancy for state $s$ and action $a$.
Our key algorithmic contribution is to design a new regularizer based on the $\nicefrac{1}{2}$-Tsallis-entropy used in~\citep{zimmert2019optimal}.
However, we argue that using only the Tsallis entropy, in the form of $-\sum_{s,a}\sqrt{q(s,a)}$, is not enough.
Instead, inspired by the work of~\citep{zimmert2019beating} for semi-bandits,
we propose to use a {\it hybrid} regularizer in the form $-\sum_{s,a}(\sqrt{q(s,a)}+\sqrt{q(s) - q(s,a)})$ where $q(s) = \sum_a q(s,a)$.
In fact, to stabilize the algorithm, we also need to add yet another regularizer in the form $-\sum_{s,a} \log q(s,a)$ (known as log-barrier), borrowing the idea from~\citep{bubeck2018sparsity, bubeck2019improved, lee2020closer}.
See \pref{sec:FTRL} and \pref{sec:algoritm_main_result} for more detailed discussions on the design of our regularizer.

More importantly, we emphasize that analyzing our new regularizer requires significantly new ideas, mainly because it admits a {\it non-diagonal Hessian} with a highly complicated inverse.
Indeed, the key of the \ftrl analysis lies in analyzing the quadratic norm of the loss estimator with respect to the inverse Hessian of the regularizer.
As far as we know, almost all regularizers used in existing \ftrl methods are decomposable over coordinates and thus admit a diagonal Hessian, making the analysis relatively straightforward (with some exceptions mentioned in related work below).
Our approach is the first to apply and analyze an explicit non-decomposable regularizer with non-diagonal Hessian.
Our analysis heavily relies on rewriting $q(s)$ in a different way and constructing the Hessian inverse recursively (see \pref{sec:analysis}).
The way we analyze our algorithm and derive a self-bounding regret bound for MDPs is the key technical contribution of this work and might be of independent interest. 

While we only resolve the problem with known transition, we believe that our approach, providing the first best-of-both-worlds result for MDPs, sheds light on how to solve the general case with unknown transition. 

\paragraph{Related work.}
We refer the reader to~\citep{simc2019} for earlier works on gap-dependent logarithmic regret bounds for learning MDPs with stochastic losses, and to~\citep{jin2019learning} for earlier works on learning MDPs with adversarial losses.
Using very different techniques, the recent work of~\citet{lykouris2019corruption} also develops an algorithm for the stochastic setting that is robust to a certain amount of adversarial corruption to the environment (including both the transition and the losses).
Their algorithm does not ensure a worst-case bound of $\order(\sqrt{T})$
and can only tolerate $o(\sqrt{T})$ amount of corruption, while our algorithm ensures $\order(\sqrt{T})$ regret always.
On the other hand, their algorithm works even under unknown transition while ours cannot.

For the special case of multi-armed bandits (essentially our setting with one single state),
the question of achieving best-of-both-worlds was first considered by~\citet{bubeck2012best}.
Since then, different improvements have been proposed over the years~\citep{seldin2014one, auer2016algorithm, seldin2017improved, wei2018more, lykouris2018stochastic, gupta2019better, zimmert2019optimal, zimmert2019beating}.
As mentioned, our regularizer is largely based on the most recent advances from~\citep{zimmert2019optimal, zimmert2019beating}, briefly reviewed in \pref{sec:FTRL}.

As far as we know, all existing works on using \ftrl with a regularizer that admits a non-diagonal Hessian do not require calculating the Hessian inverse explicitly.
For example, the SCRiBLe algorithm of~\citep{abernethy2008competing, abernethy2012interior} for efficient bandit linear optimization and its variants (e.g.,~\citep{saha2011improved, hazan2014bandit}) use any self-concordant barrier~\citep{nesterov1994interior} of the decision set as the regularizer, and the entire analysis only relies on certain properties of self-concordant barriers and does not even require knowing the explicit form of the regularizer.
As another example, when applying \ftrl over a space of matrices, the regularizer usually also does not decompose over the entries (see e.g.,~\citep{kotlowski2019bandit}).
However, the Hessian inverse is often well-known from matrix calculus already.
These previous attempts are all very different from our analysis where we need to explicitly work out the inverse of the non-diagonal Hessian.

\section{Preliminaries}
\label{sec:probform}

The problem of learning an episodic MDP through $T$ episodes is defined by a tuple $(S,A,P,L,\{\ell_t\}_{t=1}^T)$, 
where $S$ and $A$ are the finite state and action space respectively, $P:S \times A \times S \rightarrow [0,1]$ is the transition function such that $P(s'|s,a)$ is the probability of moving to state $s'$ after executing action $a$ at state $s$, $L$ is the number of interactions within each episode, and $\ell_t: S\times A\rightarrow [0,1]$ is the loss function for episode $t$, specifying the loss of each state-action pair. 

Without loss of generality (see detailed discussions in~\citep{jin2019learning}), the MDP is assumed to have the following layered structure.
First, the state space $S$ consists of $L+1$ layers $S_0, \ldots, S_L$ such that $S$ = $\bigcup_{k=0}^{L}S_k$ and $S_i \cap S_j = \emptyset$ for $i\neq j$.
Second, $S_0$ and $S_L$ are singletons, containing only the start state $s_0$ and the terminal state $s_L$ respectively.
Third, transitions are only possible between consecutive layers. In other words, if $P(s'|s,a) > 0$, then $s' \in S_{k+1}$ and $x\in S_k$ for some $k$. 

Ahead of time, the environment decides $(S,A,P,L,\{\ell_t\}_{t=1}^T)$, with $S, A, L$, and $P$ revealed to the learner.
The interaction between the learner and the environment then proceeds in $T$ episodes. 
For each episode $t$, 
the learner decides a stochastic policy $\pi_t: S \times A \rightarrow  [0,1]$, where $\pi_t(a|s)$ is the probability of selecting action $a$ at state $s$, then executes the policy starting from the initial state $s_0$ until reaching the final state $s_L$, yielding and observing a sequence of $L$ state-action-loss tuples $(s_0, a_0, \ell_t(s_0, a_0)), \ldots, (s_{L-1}, a_{L-1}, \ell_t(s_{L-1}, a_{L-1}))$, where action $a_k$ is drawn from $\pi_t(\cdot| s_k)$ and state $s_{k+1}$ is drawn from $P(\cdot|s_k,a_k)$, for each $k = 0,\ldots,L-1$. 
Importantly, the learner only observes the loss of the visited action-state pairs and nothing else about the loss function $\ell_t$, known as the {\it bandit feedback} setting.

For any policy $\pi$, with slight abuse of notation we denote its expected loss in episode $t$ as 
$
\ell_t(\pi) = \E\big[\sum_{k=0}^{L-1}\ell_t(s_k,a_k)\big],
$
where the state-action pairs $(s_0, a_0), \ldots, (s_{L-1}, a_{L-1})$ are generated according to the transition function $P$ and the policy $\pi$. 
The expected regret of the learner with respect to a policy $\pi$ is defined as $\Reg_T(\pi) = \E\big[\sum_{t=1}^T \ell_t(\pi_t) - \sum_{t=1}^T \ell_t(\pi)\big]$,
which is the difference between the total loss of the learner and that of policy $\pi$.
The goal of the leaner is to minimize her regret with respect to an optimal policy, denoted as $\Reg_T =\max_\pi \Reg_T(\pi)$.
Throughout the paper, we use $\optpi: S\rightarrow A$ to denote a {\it deterministic} optimal policy, which is known to always exist.

\paragraph{Occupancy measures.}
To solve the problem using techniques from online learning, we need the concept of ``occupancy measures'' proposed in \citep{zimin2013}.
Specifically, fixing the MDP of interest, every stochastic policy $\pi$ induces an occupancy measure $q^\pi: S \times A \rightarrow [0,1]$ with $q^\pi(s, a)$ being the probability of visiting state-action pair $(s,a)$ by following policy $\pi$.
With this concept, the expected loss of a policy $\pi$ in episode $t$ can then be written as a simple linear function of $q^\pi$:
$\ell_t(\pi) = \sum_{s\neq s_L, a\in A}q^\pi(s,a)\ell_t(s,a)$, which we denote as $\inner{q^\pi, \ell_t}$,
and the regret can be written as $\Reg_T = \E\big[\sum_{t=1}^T \inner{q_t - \opt, \ell_t}\big]$ where $q_t = q^{\pi_t}$ and $\opt = q^{\optpi}$.

In other words, the problem essentially becomes an instance of online linear optimization with bandit feedback,
where in each episode the learner proposes an occupancy measure $q_t \in \Omega$.
Here, $\Omega$ is the set of all possible occupancy measures and is known to be a polytope satisfying two constraints~\citep{zimin2013}:
first, for every $k = 0, \ldots, L-1$, $\sum_{s\in S_k}\sum_{a\in A}q(s,a) = 1$; 
second, for every $k = 1, \ldots, L-1$ and every state $s\in S_k$,
\begin{equation}
\sum_{s'\in S_{k-1}}\sum_{a'\in A}q(s',a')P(s|s',a') = \sum_{a}q(s,a).
\label{eq:occupancy_measure_cond_2}
\end{equation}
Having an occupancy measure $q_t \in \Omega$, one can directly find its induced policy via $\pi_t(a | s) = q_t(s,a)/\sum_{a'\in A} q_t(s,a')$.

\paragraph{More notation.} 
We use $k(s)$ to denote the index of the layer to which state $s$ belongs, and $\Ind{\cdot}$ to denote the indicator function whose value is $1$ if the input holds true and $0$ otherwise.
For a positive definite matrix $M \in \fR^{ K \times K}$, $\norm{x}_M \triangleq \sqrt{ x^\top M x}$ is the quadratic norm of $x\in \fR^{K}$ with respect to $M$. 
Throughout the paper, we use $\otil(\cdot)$ to hide terms of order $\log T$.

\subsection{Conditions on loss functions}
\label{sec:preliminaries_self_bounding}

So far we have not discussed how the environment decides the loss functions $\ell_1, \ldots, \ell_T$.
We consider two general settings.
The first one is the {\it adversarial setting}, where there is no assumption at all on how the losses are generated --- they can even be generated in a malicious way after seeing the learner's algorithm (but not her random seeds).\footnote{%
Technically, this is a setting with an {\it oblivious} adversary, where $\ell_t$ does not depend on the learner's previous actions. However, this is only for simplicity and our results generalize to adaptive adversaries directly. 
}
The O-REPS algorithm of~\citep{zimin2013} achieves $\Reg_T = \otil(\sqrt{L|S||A|T})$ in this case, which was shown to be optimal.

The second general setting we consider subsumes many cases such as the stochastic case.
Generalizing~\citep{zimmert2020optimal} (the full version of~\citep{zimmert2019optimal}) for bandits, we propose to summarize this setting by the following condition on the losses:
there exist a gap function $\gap: S\times A \rightarrow \fR_+$, a mapping $\pi^\star: S\rightarrow A$, and a constant $C\geq 0$ such that
\begin{equation}
\label{eq:self_bounding_constraint}
\Reg_T = \E\sbr{\sum_{t=1}^T \inner{q_t - \opt, \ell_t}} \geq \E\sbr{\sum_{t=1}^{T}\sum_{s\neq s_L}\sum_{a\neq \pi^\star(s)}q_t(s,a)\gap\rbr{s,a}}
- C 
\end{equation}
holds for all sequences of $q_1, \ldots, q_T$.
While seemingly strong and hard to interpret, this condition in fact subsumes many existing settings as explained below.

\paragraph{Stochastic losses.}
In the standard stochastic setting studied by most works in the literature,
$\ell_1, \ldots, \ell_T$ are i.i.d. samples of a fixed and unknown distribution. 
In this case, by the performance difference lemma (see e.g.,~\citep[Lemma~5.2.1]{kakade2003sample}), \pref{eq:self_bounding_constraint} is satisfied with $\pi^\star = \optpi$, $C = 0$, and $\gap(s,a) = Q(s, a) - \min_{a'\neq a}Q(s,a')$ where $Q$ is the Q function of the optimal policy $\optpi$.\footnote{%
One caveat here is that since we require $\gap(s,a)$ to be non-zero, the optimal action for each state needs to be unique for \pref{eq:self_bounding_constraint} to hold. 
The work of~\citep{zimmert2020optimal} requires this uniqueness condition for multi-armed bandits as well, which is conjectured to be only an artifact of the analysis.
}
In fact, \pref{eq:self_bounding_constraint} holds with equality in this case.
Here, $\Delta(s,a)$ is the sub-optimality gap of action $a$ at state $s$, and plays a key role in the optimal logarithmic regret for learning MDPs as shown in~\citep{simc2019}.

\paragraph{Stochastic losses with corruption.}
More generally, consider a setting where the environment first generates $\ell_1', \ldots, \ell_T'$ as i.i.d. samples of an unknown distribution (call the corresponding MDP $\calM$), and then corrupt them in an arbitrary way to arrive at the final loss functions $\ell_1, \ldots, \ell_T$.
Then \pref{eq:self_bounding_constraint} is satisfied with $\pi^\star$ being the optimal policy of $\calM$, $C = 2\sum_{t=1}^T\sum_{k<L} \max_{s\in S_k, a}|\ell_t(s,a) - \ell_t'(s, a)|$ being the total amount of corruption, and $\gap(s,a) = Q(s, a) - \min_{a'\neq a}Q(s,a')$ where $Q$ is the Q function of policy $\pi^\star$ with respect to $\calM$.
This is because:
$
\Reg_T \geq 
\sum_{t=1}^T \inner{q_t - q^{\pi^\star}, \ell_t}
= \inner{q_t - q^{\pi^\star}, \ell_t'} + \inner{q_t , \ell_t - \ell_t'}
- \inner{q^{\pi^\star}, \ell_t - \ell_t'} 
\geq \sum_{t=1}^{T}\sum_{s\neq s_L}\sum_{a\neq \pi^\star(s)}q_t(s,a)\gap\rbr{s,a} - C,
$
where in the last step we use the performance difference lemma again and the definition of $C$.

Compared to the corruption model studied by~\citep{lykouris2019corruption}, the setting considered here is more general in the sense that $C$ measures the total amount of corruption in the loss functions, instead of the number of corrupted episodes  as in~\citep{lykouris2019corruption}.
On the other hand, they allow corrupted transition as well when an episode is corrupted, while our transition is always fixed and known.

%

\subsection{Follow-the-Regularized-Leader and self-bounding regret}\label{sec:FTRL} 

$\ftrl$ is one of the standard frameworks to derive online learning algorithms for adversarial environments.
In our context, $\ftrl$ computes $q_t = \argmin_{q\in\Omega} \sum_{\tau<t} \big\langle q, \hatl_\tau\big\rangle + \psi_t(q)$ where $\hatl_\tau$ is some estimator for the loss function $\ell_\tau$ and $\psi_t$ is a regularizer usually of the form $\psi_t = \frac{1}{\eta_t}\psi$ for some learning rate $\eta_t > 0$ and a fixed convex function $\psi$.
The O-REPS algorithm exactly falls into this framework with $\hatl_t$ being the importance-weighted estimator (more details in \pref{sec:algoritm_main_result}) and $\psi$ being the entropy function.

While traditionally designed for adversarial environments, somewhat surprisingly $\ftrl$ was recently shown to be able to adapt to stochastic environments with $\order(\log T)$ regret as well, in the context of multi-armed bandits starting from the work by~\citet{wei2018more}, which was later greatly improved by~\citet{zimmert2019optimal, zimmert2020optimal}.
This approach is conceptually extremely clean and only relies on designing a regularizer that enables a certain kind of self-bounding regret bound.
We briefly review the key ideas below since our approach is largely based on extending the same idea to MDPs.

Multi-armed bandit is a special case of our setting with $L=1$.
Thus, the concept of states does not play a role, and below we write $q(s_0, a)$ as $q(a)$ for conciseness.
The regularizer used in~\citep{zimmert2019optimal} is the $\nicefrac{1}{2}$-Tsallis-entropy, originally proposed in~\citep{audibert2009minimax}, and is defined as $\psi(q) = -\sum_a \sqrt{q(a)}$.
With a simple learning rate schedule $\eta_t = 1/\sqrt{t}$, it was shown that \ftrl ensures the following adaptive regret bound for some constant $\cnt >0$,
\begin{equation}
\label{eq:mab_main_result}
\Reg_T \leq \E\sbr{\cnt \sum_{t=1}^{T} \sum_{a\neq a^\star} \sqrt{\frac{q_t(a)}{t}}} 
\end{equation}  
where $a^\star$ can be {\it any} action.
The claim is now that, solely based on \pref{eq:mab_main_result}, one can derive the best-of-both-worlds guarantee already, without even further considering the details of the algorithm.
To see this, first note that
applying Cauchy-Schwarz inequality ($\sum_a\sqrt{q_t(s)} \leq \sqrt{|A|}$) immediately leads to a worst-case robustness guarantee of $\Reg_T = \order(\sqrt{|A|T})$ (optimal for multi-armed bandits).

More importantly, suppose now Condition~\eqref{eq:self_bounding_constraint} holds 
(note again that there is only one state $s=s_0$ and we write $\gap(s_0,a)$ as $\gap(a)$).
Then picking $a^\star = \pi^\star(s_0)$, one can arrive at the following self-bounding regret using \pref{eq:mab_main_result}:
\begin{equation*}
\Reg_T  \leq \E\sbr{\sum_{t=1}^{T} \sum_{a\neq a^\star} \frac{q_t(a)\gap(a)}{2z} + \frac{z\cnt^2}{2t\gap(a)}} \leq \frac{\Reg_T+C}{2z} + z\cnt^2\sum_{a\neq a^\star}\frac{\log T}{\gap(a)}, 
\end{equation*}
where the first step uses the AM-GM inequality and holds for any $z>1/2$, and the second step uses \pref{eq:self_bounding_constraint} and the fact $\sum_{t=1}^T 1/t \leq 2\log T$.
Note that we have bounded the regret in terms of itself (hence the name self-bounding).
Rearranging then gives $\Reg_T \leq \frac{2z^2\cnt^2}{2z-1}\sum_{a\neq a^\star}\frac{\log T}{\gap(a)} + \frac{C}{2z-1}$.
It just remains to pick the optimal $z$ to minimize the bound.
Specifically, using the shorthand $U = \frac{1}{2}\cnt^2\sum_{a\neq a^\star}\frac{\log T}{\gap(a)}$ and $x=2z-1 >0$,
the bound can be simplified to $\Reg_T \leq 2U + Ux + (C+U)/x$,
and finally picking the best $x$ to balance the last two terms gives $\Reg_T \leq 2U + 2\sqrt{U(C+U)} \leq 4U + 2\sqrt{UC} = 2\cnt^2\sum_{a\neq a^\star}\frac{\log T}{\gap(a)} + \sqrt{2C\cnt^2\sum_{a\neq a^\star}\frac{\log T}{\gap(a)}}$.

In the case with stochastic losses ($C=0$), the final bound is $\order\big(\sum_{a\neq a^\star}\frac{\log T}{\gap(a)}\big)$, exactly matching the lower bound for stochastic multi-armed bandits~\citep{lai1985asymptotically}.
More generally, in the corruption model, the regret is order $\order(\log T + \sqrt{C\log T})$, smoothly interpolating between the bounds for the stochastic setting and the adversarial setting as $C$ increases from $0$ to $T$.

Finally, we remark that although not mentioned explicitly, the follow-up work~\citep{zimmert2019beating} reveals that using a {\it hybrid} regularizer in the form $\psi(q) = -\sum_a (\sqrt{q(a)} + \sqrt{1-q(a)})$, with Tsallis entropy applied to both $q$ and its complement, also leads to the same bound \pref{eq:mab_main_result}, via an even simpler analysis.
This is crucial for our algorithm design and analysis as explained in the next section.

\section{Algorithm and Main Results}
\label{sec:algoritm_main_result}

\begin{algorithm}[t]
	\caption{\ftrl with hybrid Tsallis entropy for learning stochastic and adversarial MDPs}
	\label{alg:main_alg}
	\begin{algorithmic}
		\STATE {\bfseries Parameters:} $\alpha$, $\paralog$, $\gamma$
		\STATE {\bfseries Define:} hybrid regularizer $\phi_H$ and log-barrier regularizer $\phi_L$ as in \pref{eq:hybrid} and \pref{eq:log-barrier}
		\STATE {\bfseries Define:} valid occupancy measure space $\Omega$ (see \pref{sec:probform}), learning rate $\eta_t = \gamma /\sqrt{t}$
		\STATE {\bfseries Initialize:} $\hatL_0(s,a) = 0$ for all $(s,a)$
		\FOR{$t=1$ {\bfseries to} $T$}
		\STATE compute $q_t = \argmin_{q\in\Omega} \big\langle q, \hatL_{t-1}\big\rangle + \psi_t(q)$ where $\psi_t(q) = \frac{1}{\eta_t} \phi_H(q) + \phi_L(q)$
		\STATE execute policy $\pi_t$ where $\pi_t(a | s) = q_t(s,a)/q_t(s)$
		\STATE observe $(s_0, a_0, \ell_t(s_0, a_0)), \ldots, (s_{L-1}, a_{L-1}, \ell_t(s_{L-1}, a_{L-1}))$
		\STATE construct estimator $\hatl_t$ such that:  $\forall (s,a), \hatl_t(s,a) = \frac{\ell_t(s,a)}{q_t(s,a)}\Ind{s_{k(s)}=s,a_{k(s)}= a}$ 
		\STATE update $\hatL_t = \hatL_{t-1} + \hatl_t$
		\ENDFOR
	\end{algorithmic}
\end{algorithm}

We are now ready to present our algorithm.
Based on the discussions from \pref{sec:FTRL}, our goal is to design an algorithm with a regret bound akin to \pref{eq:mab_main_result}: 
\begin{equation}
\label{eq:main_goal}
\Reg_T \leq \E\sbr{\cnt\sum_{t=1}^{T} \sum_{s\neq s_L} \sum_{a\neq \pi(s)} \sqrt{\frac{q_t(s,a)}{t}}}
\end{equation}
for {\it any} mapping $\pi:  S\rightarrow A$.
This immediately implies a worst-case regret bound $\Reg_T = \order(\sqrt{L|S||A|T})$ (by applying Cauchy-Schwarz inequality), matching the optimal bound of~\citep{zimin2013}.
Moreover, repeating the same calculation and picking $\pi = \pi^\star$, one can verify that under Condition~\eqref{eq:self_bounding_constraint} this leads to a similar bound of order $\order(\log T + \sqrt{C\log T})$ as in the case for multi-armed bandits.

To achieve so, a natural idea is to directly extend the $\nicefrac{1}{2}$-Tsallis-entropy regularizer to MDPs and take $\psi(q) = -\sum_{s,a}\sqrt{q(s,a)}$.
However, generalizing the proofs of~\citep{zimmert2019optimal} one can only prove a weaker bound:
$\Reg_T \leq \cnt\sum_{t=1}^{T} \sum_{(s,a)\neq \xi(k(s))} \sqrt{\frac{q_t(s,a)}{t}}$
for any $\xi$ mapping from a layer index to a state-action pair.
Compared to the desired~\pref{eq:main_goal}, one can see that instead of excluding one arbitrary action $\pi(s)$ for {\it each} state $s$ in the bound, now we only exclude one arbitrary action for {\it one single} state specified by $\xi$ in each layer.
This is not enough to derive the same results as one can verify.

The hybrid regularizer $-\sum_{s,a}(\sqrt{q(s,a)} + \sqrt{1-q(s,a)})$~\citep{zimmert2019beating} suffers the same issue.
However, we propose a natural fix to this hybrid version by replacing $1$ with $q(s) \triangleq \sum_{a} q(s,a)$, that is, the marginal probability of visiting state $s$ under occupancy measure $q$.\footnote{%
We find it intriguing that while the Tsallis entropy regularizer and its hybrid version work equally well for multi-armed bandits, only the latter admits a natural way to be generalized to learning MDPs.
}
More concretely, we define (with ``$H$'' standing for ``hybrid'')
\begin{equation}\label{eq:hybrid}
\phi_{H}(q) = - \sum_{s\neq s_L,a\in A} \rbr{\sqrt{ q(s,a) } + \alpha \sqrt{ q(s) - q(s,a)}}, 
\quad\text{where}\;\; q(s) = \sum_{a\in A} q(s,a)
\end{equation}
for some parameter $\alpha >0$, as the key component of our regularizer.
Note that in the case of multi-armed bandits, this exactly recovers the original hybrid regularizer since $q(s_0)=1$ and there is only one state.
For MDPs, intuitively each state is dealing with a multi-armed bandit instance, but with total available probability $q(s)$ instead of $1$, making $\phi_H$ a natural choice.
However, also note another important distinction between $\phi_H$ and the ones discussed earlier: $\phi_H$ does not decompose over the action-state pairs, thus admitting a {\it non-diagonal Hessian}.
This makes the analysis highly challenging since standard \ftrl analysis requires analyzing the Hessian inverse of the regularizer.
We will come back to this challenge in \pref{sec:analysis}.

To further stabilize the algorithm and make sure that $q_t$ and $q_{t+1}$ are not too different (another important requirement of typical \ftrl analysis), we apply another regularizer in addition to $\phi_H$, defined as (with ``$L$'' standing for ``log-barrier''):
\begin{equation}\label{eq:log-barrier}
\phi_{L}(q) = \paralog \sum_{s\neq s_L,a\in A} \log\frac{1}{q(s,a)},  
\end{equation}
for some parameter $\paralog > 0$.
Our final regularizer for time $t$ is then $\psi_t(q) = \frac{1}{\eta_t} \phi_H(q) + \phi_L(q)$ where $\eta_t = \gamma / \sqrt{t}$ is a decreasing learning rate with parameter $\gamma > 0$.
In all our results we pick $\paralog = \order(L)$.
Thus, the weight for $\phi_L$ is much smaller than that for $\phi_H$.
This idea of adding a small amount of log-barrier to stabilize the algorithm was first used in~\citep{bubeck2018sparsity} and recently applied in several other works~\citep{bubeck2019improved, zheng2019equipping, lee2020closer}.
See more discussions in \pref{sec:analysis}.

Our final algorithm is shown in \pref{alg:main_alg}.
In each episode $t$, the algorithm follows standard \ftrl and computes $q_t = \argmin_{q\in\Omega} \big\langle q, \hatL_{t-1}\big\rangle + \psi_t(q)$ with $\hatL_{t-1} = \sum_{s<t}\hatl_s$ being the accumulated estimated loss.
Then the policy $\pi_t$ induced from $q_t$ is executed, generating a sequence of state-action-loss tuples.
A standard importance-weighted unbiased estimator $\hatl_t$ is then constructed with $\hatl_t(s,a)$ being the actual loss $\ell_t(s,a)$ divided by $q_t(s,a)$ if the state-action pair $(s,a)$ was visited in this episode, and zero otherwise.
Also note that \pref{alg:main_alg} can be efficiently implemented since the key \ftrl step is a convex optimization problem with $\order(L+|S||A|)$ linear constraints (solving it to an inaccuracy of $\order(1/T)$ is enough clearly).

\subsection{Main Results}
\label{sec:result}

We move on to present the regret guarantees of our algorithm.
As mentioned, the goal is to show \pref{eq:main_goal},
and the theorem below shows that our algorithm essentially achieves this (see \pref{app:appendix_theorem_main} for the proof).

\begin{theorem}
	\label{thm:main}
	With $\alpha = \nicefrac{1}{\sqrt{|A|}}$, $\paralog=64L$, and $\gamma = 1$, \pref{alg:main_alg} ensures
	that $\Reg_T$ is bounded by
\begin{equation}\label{eq:main}
\sum_{t=1}^T \otil\rbr{\min\cbr{
\E\sbr{\cnt\sum_{s\neq s_L}\sum_{a\neq \pi(s)}\sqrt{\frac{q_t(s,a)}{t}} + D\sqrt{\sum_{s\neq s_L}\sum_{a\neq \pi(s)}\frac{q_t(s,a)+\opt(s,a)}{t}}},
\frac{D}{\sqrt{t}}
}
}
\end{equation}	
for any mapping $\pi: S \rightarrow A$,
where $\cnt = L^\frac{5}{2}+L\sqrt{|A|}$ and $D = \sqrt{L|S||A|}$.
\end{theorem}

Looking at the first term of the min operator, one sees that we have an extra term compared to the ideal bound \pref{eq:main_goal}.
However, this only contributes to small terms in the final bounds as we explain below.\footnote{%
The fact that this form of weaker bounds is also sufficient for the self-bounding property might be of interest for other bandit problems as well.}
Based on this theorem, we next present more concrete regret bounds for our algorithm.
First, consider the adversarial setting where there is no further structure in the losses.
Simply taking the second term of the min operator in \pref{eq:main} we obtain the following corollary.

\begin{corollary} 
	\label{col:adversarial_bound}
	With $\alpha = \nicefrac{1}{\sqrt{|A|}}$, $\paralog=64L$, and $\gamma = 1$, the regret of \pref{alg:main_alg} is always bounded as 
	\begin{equation*}
	\Reg_T \leq \otil\rbr{\sqrt{L|S||A|T}}.
	\end{equation*}
\end{corollary}

Again, this bound matches that of the O-REPS algorithm~\citep{zimin2013} and is known to be optimal.
Note that using the first term of the min operator we can also derive an $\otil(\sqrt{T})$ bound, but the dependence on other parameters would be larger.

On the other hand, under Condition~\pref{eq:self_bounding_constraint}, we obtain the following bound.
\begin{corollary} \label{col:stochastic_bound} 
Suppose Condition~\eqref{eq:self_bounding_constraint} holds. Then with $\alpha = \nicefrac{1}{\sqrt{|A|}}$, $\paralog=64L$, and $\gamma = 1$, the regret of Algorithm~\ref{alg:main_alg} is bounded as $\Reg_T \leq \order(U + \sqrt{UC})$ where 
\[
U = \frac{L|S||A| \log T}{\gapmin} + L^2\rbr{L^3 + |A|}\sum_{s\neq s_L}\sum_{a \neq \pi^\star(s) }\frac{\log T}{ \gap(s,a) }
= \order(\log T)
\]
and $\gapmin =\min_{s\neq s_L, a\neq \pi^\star(s)} \gap(s,a)$ is the minimal gap.
Consequently, in the stochastic setting, we have $\Reg_T = \order(U)$.
\end{corollary}


See \pref{app:appendix_corollary_stochastic} for the proof, whose idea is similar to the discussions in \pref{sec:FTRL}.
Note that we are getting an extra term in $U$ involving $1/\gapmin$, which in turn comes from the extra term in \pref{eq:main} mentioned earlier.
For the stochastic setting, the best existing logarithmic bound is $\order\big(L^3\sum_{s\neq s_L}\sum_{a \neq \pi^\star(s)}\frac{\log T}{ \gap(s,a) } \big)$ from~\citep{simc2019} (which also has a dependence on $\frac{1}{\gapmin}$ but is ignored here for simplicity).
Our bound has a worse factor $L^2(L^3+|A|)$.
We note that by tuning the parameters $\alpha$ and $\gamma$ differently, one can obtain a better bound in this case.
This set of parameters, however, leads to a sub-optimal adversarial bound.
See \pref{app:appendix_parameter_choices} for details.
Since our goal is to develop one single algorithm that adapts to different environments automatically,
we stick to the same set of parameters in the theorem and corollaries above.

As mentioned, in the corruption setting, our bound $\otil(\sqrt{C})$ smoothly interpolates between the logarithmic regret in the no corruption case and the $\otil(\sqrt{T})$ worst-case regret in the adversarial case.
On the other hand, the bound from~\citep{lykouris2019corruption} is of order $\otil(C^2)$, only meaningful for $C = o(\sqrt{T})$.

\section{Analysis Sketch}
\label{sec:analysis}

Analyzing our algorithm requires several new ideas, which we briefly mention in this section through a few steps, with all details deferred to the appendix.

\paragraph{First step.}
While we define $q(s)$ as $\sum_a q(s,a)$, the entire analysis relies on using an equivalent definition of $q(s)$ based on \pref{eq:occupancy_measure_cond_2}.
The benefit of this alternative definition is that it contains the information of the transition function $P$ and implicitly introduces a layered structure to the regularizer, which is important for constructing the Hessian and its inverse recursively.
We emphasize that, however, this does {\it not} change the algorithm at all, because all occupancy measures in $\Omega$ ensure \pref{eq:occupancy_measure_cond_2} by definition and the \ftrl optimization over $\Omega$ is thus not affected.

\paragraph{Second step.}
Following the standard $\ftrl$ analysis one can obtain a regret bound in terms of $\|\hatl_t\|_{\nabla^{-2}\phi_{H}(q_t')}$ for some $q_t'$ between $q_t$ and $q_{t+1}$.
Moving from $q_t'$ to $q_t$ is exactly the part where the log-barrier $\phi_L$ is important.
Specifically, following the idea of~\citep[Lemma~9]{lee2020closer}, we prove that $q_t$ and $q_{t+1}$ are sufficiently close, and consequently $\Reg_T$ is mainly bounded by two terms:
the penalty term $\sum_{t=1}^{T}(\nicefrac{1}{\eta_t}-\nicefrac{1}{\eta_{t-1}})\rbr{\phi_{H}(\opt) -\phi_{H}(q_t)}$
and the stability term $\sum_{t=1}^{T}\eta_t \|\hatl_t\|^2_{\nabla^{-2}\phi_{H}(q_t)}$ (see \pref{lem:key_lemma_ftrl}).


\paragraph{Third step.}
Bounding the penalty term already requires a significant departure from the analysis for multi-armed bandits.
Specifically,  $\phi_{H}(\opt) -\phi_{H}(q_t)$ can be written as
\begin{equation}\label{eq:penalty_decomposition}
\sum_{s\neq s_L}\sqrt{q_t(s)}\rbr{h_s(\pi_t)-h_s(\optpi)} + \sum_{s \neq s_L}\rbr{\sqrt{q_t(s)}-\sqrt{\opt(s)}}h_s(\optpi)
\end{equation} 
where $h_s(\pi)= \sum_{a\in A}\sqrt{\pi(a|s)}+\alpha\sqrt{1-\pi(a|s)}$ is basically the hybrid regularizer for  multi-armed bandits (at state $s$) mentioned in \pref{sec:FTRL}.
The first term in \pref{eq:penalty_decomposition} can then be bounded as $(1+\alpha) \sum_{s}\sum_{a \neq \pi(s)}\sqrt{q_t(s,a)}$ for {\it any} mapping $\pi: S\rightarrow A$, in a similar way as in the multi-armed bandit analysis.
However, the key difficulty is the second term, which does not appear for multi-armed bandits where there is only one state $s_0$ with $q(s_0)=1$ for all $q$.
This is highly challenging especially because we would like to arrive at a term with a summation over $a \neq \pi(s)$ as in the first term.
Our main idea is to separately bound $\sqrt{q_t(s)}-\sqrt{q^{\pi}(s)}$ and $\sqrt{q^{\pi}(s)}-\sqrt{\opt(s)}$
via a key induction lemma (\pref{lem:induction_state_reach_prob}) that connects them to similar terms in previous layers.
We remark that this term is the source of the extra term $\sqrt{\sum_{s\neq s_L}\sum_{a\neq \pi(s)}(q_t(s,a)+\opt(s,a))/t}$ in our main regret bound \pref{eq:main}.
The complete proof for bounding the penalty term is in \pref{app:appendix_proof_penalty}.

\paragraph{Fourth step.}
Analyzing the stability term is yet another highly challenging part, since it requires working out the Hessian inverse $\nabla^{-2}\phi_{H}(q_t)$.
The high-level idea of our proof is to first write the Hessian in a recursive form based on the layered structure introduced by writing $q(s)$ differently as mentioned in the first step.
Then we apply Woodbury matrix identity to obtain a recursive form of the Hessian inverse.
Finally, we argue that only certain parts of the Hessian inverse matter, and these parts enjoy nice properties allowing us to eventually bound $\E\sbr{\|\hatl_t\|^2_{\nabla^{-2}\phi_{H}(q_t)}}$ by $8e L^2\rbr{\sqrt{L} + \nicefrac{1}{\alpha L} }\sum_{s\neq s_L}\sum_{a\neq \pi(s)}\sqrt{q_t(s,a)}$, 
again, for any mapping $\pi$.
See \pref{app:appendix_proof_stability} for the complete proof.

\section{Conclusion}
\label{sec:conclusion}

In this work, we provide the first algorithm for learning episodic MDPs that automatically adapts to different environments with favorable guarantees.
Our algorithm applies a natural regularizer with a complicated non-diagonal Hessian to the \ftrl framework,
and our analysis for obtaining a self-bounding regret bound requires several novel ideas.
Apart from improving our bound in \pref{col:stochastic_bound},
one key future direction is to remove the known transition assumption, which appears to demand new techniques since it is hard to also control the bias introduced by estimating the transition in terms of the adaptive bound in \pref{eq:main}.

\section*{Broader Impact}
\label{sec:broader_impact}

This work is mostly theoretical, with no negative outcomes.
Researchers working on theoretical aspects of online learning, bandit problems, and reinforcement learning (RL) may benefit from our results.
Although our algorithm deals with the tabular setting and is not directly applicable to common RL applications with a large state and action space, 
it sheds light on how to increase robustness of a learning algorithm while adapting to specific instances,
and serves as an important step towards developing more practical, adaptive, and robust RL algorithms, which in the long run might find its applications in the real world.

%

\begin{ack}
We thank Max Simchowitz for many helpful discussions. 
HL is supported by NSF Awards IIS-1755781 and IIS-1943607, and a Google Faculty Research Award.
\end{ack}

\bibliography{ref}
\bibliographystyle{plainnat}
\newpage
\appendix

\section{Proof of \pref{thm:main} and \pref{col:stochastic_bound}}
\label{app:appendix_theorem_main}

In this section, we provide the proof (outline) of \pref{thm:main}, the proof of \pref{col:stochastic_bound} (\pref{app:appendix_corollary_stochastic}), discussions on parameter tuning (\pref{app:appendix_parameter_choices}), and also some preliminaries on the Hessian of our regularizer which are useful for the rest of the appendix (\pref{app:Hessian}).

We prove the following version of \pref{thm:main} with general value of the parameters $\alpha$ and $\gamma$, which facilitates further discussions on parameter tuning.
Taking $\alpha = 1/\sqrt{|A|}$ and $\gamma=1$ clearly recovers \pref{thm:main}.

\begin{theorem} With $\paralog = 64L$, Algorithm~\ref{alg:main_alg} guarantees: 
	\begin{align*}
	\Reg_T &= \order\rbr{\sum_{t=1}^{T}\min\cbr{\E\sbr{X\sum_{s\neq s_L}\sum_{a\neq \pi(s)}\sqrt{\frac{q_t(s,a)}{t}} + Y\sqrt{\sum_{s\neq s_L}\sum_{a\neq \pi(s)}\frac{q_t(s,a)+\opt(s,a)}{t}}}, \frac{Z}{\sqrt{t}}}} \\
	& \quad +\order\rbr{L|S||A|\log T} 
	\end{align*}
	for any mapping $\pi: S \rightarrow A$, where coefficients $X$, $Y$ and $Z$ are defined as:
	\[
	X = \frac{1+\alpha}{\gamma} + \gamma  L^2\rbr{\sqrt{L}+\frac{1}{\alpha L}}, Y = \frac{\sqrt{L|S|}\rbr{1+\alpha|A|}}{\gamma}, Z = \sqrt{L|S||A|}\rbr{\frac{1+\alpha\sqrt{|A|}}{\gamma} + \gamma}. 
	\]
	\label{thm:main_general}
\end{theorem}

The proof of this theorem relies on the following three important lemmas, which respectively correspond to the second, third, and fourth step of the analysis sketch in \pref{sec:analysis} and are proven in \pref{app:appendix_proof_ftrl}, \pref{app:appendix_proof_penalty}, and \pref{app:appendix_proof_stability}.
The first one 
decomposes the regret into two terms (penalty and stability), with an additional small term of order $\order(L|S||A|\log T)$. 
The second one bounds the penalty term, while the third one bounds the stability term.

\begin{lemma}[Regret decomposition]
\label{lem:key_lemma_ftrl}
With $\paralog = 64L$, \pref{alg:main_alg} ensures:
\[
\begin{split}
\Reg_T & \leq \underbrace{\sum_{t=1}^{T}\rbr{\frac{1}{\eta_t} - \frac{1}{\eta_{t-1}}}\E\sbr{\phi_{H}(\opt) -\phi_{H}(q_t)}}_{penalty} + \underbrace{8\sum_{t=1}^{T}\eta_t\E\sbr{\norm{\hatl_t}^2_{\nabla^{-2}\phi_{H}(q_t)}}}_{stability} \\ & \; \;+ \order\rbr{L|S||A|\log T}.
\end{split}
\]
\end{lemma}

\begin{lemma}[Penalty] \label{lem:key_lemma_bound_penalty} 
The hybrid regularizer $\phi_{H}$ defined in \pref{eq:hybrid} ensures that $\phi_{H}(\opt) -\phi_{H}(q_t)$ is bounded by 
\[
\begin{split}
\rbr{1+\alpha}\sum_{s\neq s_L}\sum_{a\neq \pi(a)}\sqrt{q_t(s,a)} + \rbr{1 + \alpha |A|}\sqrt{|S|L}\min\cbr{1,2\sqrt{\sum_{s\neq s_L}\sum_{a\neq \pi(s)}q_t(s,a)+ \opt(s,a)} },
\end{split}
\]
for all $t=1,\ldots,T$, where $\pi$ can be any mapping from $S$ to $A$. 

\end{lemma}

\begin{lemma}[Stability]\label{lem:key_lemma_bound_stability}
 \pref{alg:main_alg} guarantees that $\E\sbr{\|\hatl_t\|^2_{\nabla^{-2}\phi_{H}(q_t)}}$ is bounded by 
\[
\min\cbr{4\sqrt{L|S||A|}  , \E\sbr{8eL^2\rbr{\sqrt{L} + \frac{1}{\alpha \cdot L} } \sum_{s\neq s_L}\sum_{a\neq \pi(a)}\sqrt{q_t(s,a)}}},
\]
for all $t=1,\ldots,T$, where $\pi$ can be any mapping from $S$ to $A$. 
\end{lemma}

\begin{proof}[Proof of \pref{thm:main_general}]	
		For notational convenience, we denote 
		\[
		J_{1}(t) = \sqrt{\frac{1}{t} },  J_2(t) = \E\sbr{\sum_{s\neq s_L}\sum_{a\neq \pi(a)}\sqrt{\frac{q_t(s,a)}{t}}},  J_3(t) = \E\sbr{\sqrt{\sum_{s\neq s_L}\sum_{a\neq \pi(s)}\frac{q_t(s,a)+ \opt(s,a)}{t} }}.
		\]
	
By \pref{lem:key_lemma_bound_penalty} and the fact $\nicefrac{1}{\eta_t} - \nicefrac{1}{\eta_{t-1}} = (\sqrt{t}-\sqrt{t-1})/\gamma = \frac{1}{\gamma(\sqrt{t}+\sqrt{t-1})} \leq \nicefrac{1}{\rbr{\gamma\sqrt{t}}}$, the penalty term can be bounded by 
\[
\order\rbr{\sum_{t=1}^{T}\E\sbr{C_1\sum_{s\neq s_L}\sum_{a\neq \pi(a)}\sqrt{\frac{q_t(s,a)}{t}}} + C_2 \min\cbr{\frac{1}{\sqrt{t}},\E\sbr{ \sqrt{\sum_{s\neq s_L}\sum_{a\neq \pi(s)}\frac{q_t(s,a)+ \opt(s,a)}{t} } } }},
\]
that is, $\order\rbr{\sum_{t=1}^{T}C_1 J_2(t) + C_2 \min\cbr{ J_1(t),  J_3(t) }}$,
where $C_1 = \frac{1+\alpha}{\gamma}$ and $C_2 = \frac{(1+\alpha|A|)\sqrt{|S|L}}{\gamma}$. 
	
On the other hand, by \pref{lem:key_lemma_bound_stability}, the stability term is bounded by 
\[
\order\rbr{\sum_{t=1}^{T}\min\cbr{\frac{C_3}{\sqrt{t}},\E\sbr{ C_4\sum_{s\neq s_L}\sum_{a\neq \pi(a)}\sqrt{\frac{q_t(s,a)}{t}}}} },
\]
that is, $\order\rbr{\sum_{t=1}^{T}\min\cbr{ C_3 J_1(t),  C_4 J_2(t) }}$,
where $C_3=\gamma\sqrt{L|S||A|}$ and $C_4 =  \gamma L^2\rbr{\sqrt{L} + \nicefrac{1}{\alpha L}}$.

Finally, we plug these bounds into \pref{lem:key_lemma_ftrl} and show that 
\begin{align*}
\Reg_T & = \order\rbr{\sum_{t=1}^{T} C_1 J_2(t) +  C_2 \min\cbr{ J_1(t),  J_3(t) }+  \min\cbr{ C_3 J_1(t),  C_4 J_2(t) }}   \\ 
& \; \; + \order\rbr{L|S||A|\log T}.
\end{align*}

Noticing that $J_2(t) = \sum_{s\neq s_L}\sum_{a\neq \pi(a)}\sqrt{\frac{q_t(s,a)}{t}} \leq \sqrt{\frac{L|S||A|}{t}} = \sqrt{L|S||A|} J_1(t) $ by Cauchy-Schwarz inequality, we further have 
\begin{align*}
\Reg_T & = \order\rbr{\sum_{t=1}^{T} C_1 J_2(t) +  C_2 \min\cbr{ J_1(t),  J_3(t) }+  \min\cbr{ C_3 J_1(t),  C_4 J_2(t) }}   \\ 
& \; \; + \order\rbr{L|S||A|\log T} \\ 
& = \order\rbr{\sum_{t=1}^{T} \min\cbr{ C_1\sqrt{L|S||A|} J_1(t),  C_1 J_2(t) } +  \min\cbr{ \rbr{C_2 + C_3} J_1(t),  C_2 J_3(t) +  C_4 J_2(t)} }  \\ 
& \; \; + \order\rbr{L|S||A|\log T} \\ 
& = \order\rbr{\sum_{t=1}^{T}\min\cbr{ \rbr{C_1\sqrt{L|S||A|} + C_2 + C_3} J_1(t),  C_2 J_3(t) +  \rbr{ C_1 + C_4 } J_2(t)} }  \\ 
& \; \; + \order\rbr{L|S||A|\log T}  .
\end{align*}

Therefore, we prove the regret bound stated in Theorem~\ref{thm:main_general} with $X=C_1+C_4$, $Y=C_2$, and $Z=C_1\sqrt{L|S||A|}+C_2+C_3 = \order\rbr{\sqrt{L|S||A|}\rbr{\frac{1+\alpha\sqrt{|A|}}{\gamma} + \gamma}}$. 
%
In particular, setting $\gamma=1$ and $\alpha = \nicefrac{1}{\sqrt{|A|}}$ exactly leads to \pref{thm:main}. 
\end{proof}


\subsection{Proof of \pref{col:stochastic_bound}}
\label{app:appendix_corollary_stochastic}
The proof mostly follows the discussions in \pref{sec:FTRL}, except that we need to deal with the extra term involving $\opt$. 
To do that, we first introduce the following important implication of  Condition~\eqref{eq:self_bounding_constraint}.
\begin{lemma} \label{lem:opt_implict_bound} Suppose Condition~\eqref{eq:self_bounding_constraint} holds. Then the optimal occupancy measure $\opt$ ensures 
	\[
	T\sum_{s\neq s_L}\sum_{a\neq \pi^\star(s)}\opt(s,a)\Delta(s,a) \leq C.
	\]		
\end{lemma}
\begin{proof}
	Simply setting $q_1 = \cdots = q_T = \opt$ in Condition~\eqref{eq:self_bounding_constraint} gives
	\[
	0 = \E\sbr{\sum_{t=1}^{T}\inner{\opt - \opt, \ell_t }} \geq T\sum_{s\neq s_L}\sum_{a\neq \pi^\star(s)}\opt(s,a)\Delta(s,a) - C,\] 
	and rearranging finishes the proof.
\end{proof}

\begin{proof}[Proof of \pref{col:stochastic_bound}]
By \pref{thm:main_general} (setting $\alpha=\nicefrac{1}{\sqrt{|A|}}$ and $\gamma =1$ as stated in \pref{col:stochastic_bound} and picking $\pi = \pi^\star$), $\Reg_T$ is bounded by 
\[
\kappa\rbr{ L|S||A|\log T +
		\sum_{t=1}^{T}\E\sbr{ \cnt\sum_{s\neq s_L}\sum_{a\neq \pi^\star(s)}\sqrt{\frac{q_t(s,a)}{t}} +  D\sqrt{\sum_{s\neq s_L}\sum_{a\neq \pi^\star(s)}\frac{q_t(s,a)+\opt(s,a)}{t}}}}
\]
where $\kappa$ is a constant, $\cnt = L^\frac{5}{2}+L\sqrt{|A|}$ and $D = \sqrt{L|S||A|}$.

For any $z>0$, we have
\[
\begin{split}
& \kappa \sum_{t=1}^{T}\E\sbr{ \cnt\sum_{s\neq s_L}\sum_{a\neq \pi^\star(s)}\sqrt{\frac{q_t(s,a)}{t}}} \\ 
& \leq \sum_{t=1}^{T}\E\sbr{ \sum_{s\neq s_L}\sum_{a\neq \pi^\star(s)}\sqrt{\frac{q_t(s,a) \gap(s,a)}{z} \cdot \frac{z \kappa^2 \cnt^2 }{t \gap(s,a)} }   } \\
& \leq \E\sbr{\sum_{t=1}^{T}\sum_{s\neq s_L}\sum_{a\neq \pi^\star(s)}\rbr{\frac{q_t(s,a)\gap(s,a)}{2z} + \frac{z\kappa^2\cnt^2}{2t\gap(s,a)}}} \\ 
& \leq \frac{\Reg_T + C}{2z} + z\kappa^2L^2(L^3+|A|)\sum_{s\neq s_L}\sum_{a\neq \pi^\star(s)}\frac{\log T}{\gap(s,a)}
\end{split} 
\]
where the third line uses the AM-GM inequality, and the last line uses \pref{eq:self_bounding_constraint} and the fact $\sum_{t=1}^{T}\nicefrac{1}{t} \leq 2\log T$. 

For the other term,  we have for any $z > 0$:
\[
\begin{split}
& \kappa \sum_{t=1}^{T}\E\sbr{  D\sqrt{\sum_{s\neq s_L}\sum_{a\neq \pi^\star(s)}\frac{q_t(s,a)+\opt(s,a)}{t}}} \\
& = \sum_{t=1}^{T}\E\sbr{  \sqrt{\rbr{\sum_{s\neq s_L}\sum_{a\neq \pi^\star(s)}\frac{\rbr{q_t(s,a)+\opt(s,a)}\gapmin}{z}}  \cdot \frac{\kappa^2 D^2}{t\gapmin} }   } \\
&\leq \E\sbr{\sum_{t=1}^{T}\rbr{\sum_{s\neq s_L}\sum_{a\neq \pi^\star(s)}\frac{(q_t(s,a)+\opt(s,a))\gapmin}{2z}} + \frac{z\kappa^2D^2}{2t\gapmin}} \\
&\leq \E\sbr{\sum_{t=1}^{T}\rbr{\sum_{s\neq s_L}\sum_{a\neq \pi^\star(s)}\frac{(q_t(s,a)+\opt(s,a))\gap(s,a)}{2z}} + \frac{z\kappa^2D^2}{2t\gapmin}} \\
&\leq \frac{\Reg_T + 2C}{2z} + \frac{z\kappa^2L|S||A|\log T}{\gapmin}
\end{split}
\] 
where the third line uses the AM-GM inequality again, the third line uses the definition of $\gapmin$, and the last line uses \pref{eq:self_bounding_constraint}, \pref{lem:opt_implict_bound}, and the fact $\sum_{t=1}^{T}\nicefrac{1}{t}\leq 2\log T$.

Combining the inequalities, we have  
\[
\Reg_T \leq \frac{\Reg_T}{z} + \frac{2C}{z} + z\kappa^2 U + \kappa V 
\]
where we use the shorthand $U$ (already defined in the statement of \pref{col:stochastic_bound}) and $V$ as 
\[
U = L^2(L^3+|A|)\sum_{s\neq s_L}\sum_{a\neq \pi^\star(s)}\frac{\log T}{\gap(s,a)} + \frac{L|S||A|\log T}{\gapmin}, \quad V = L|S||A|\log T. 
\]

For any $z>1$, we can further rearrange and arrive at
\[
\begin{split}
\Reg_T & \leq \frac{2}{(z-1)}C + \frac{z^2}{z-1}\kappa^2 U + \frac{z}{z-1}\kappa V \\
& = \frac{2}{x}C + \frac{(x+1)^2}{x} \kappa^2 U + \frac{x+1}{x}\kappa V \\
& = \frac{1}{x}\rbr{ 2C + \kappa V + \kappa^2 U} + x\rbr{\kappa^2 U} + \rbr{2 \kappa^2 U + \kappa V }
\end{split}
\]
where we define $x=z-1>0$ and replace all $z$'s in the second line. 
Picking the optimal $x$ to balance the first terms gives 
\[
\begin{split}
\Reg_T & \leq 2\sqrt{\rbr{ 2C + \kappa V + \kappa^2 U} \rbr{\kappa^2 U} } + 2 \kappa^2 U + \kappa V  \\
& \leq 2\kappa \sqrt{2UC} + 2\sqrt{\kappa^3 UV} + 4\kappa^2 U + \kappa V\\
& \leq 2\kappa \sqrt{2UC} + \sqrt{\kappa^3} \rbr{ U + V} + 4\kappa^2 U + \kappa V \\
& \leq \order\rbr{ U + V + \sqrt{UC} }
\end{split}
\]
where the second line follows from the fact $\sqrt{a+b}\leq \sqrt{a}+\sqrt{b}$, and the third line uses AM-GM inequality. 
Finally, noticing that $V \leq  \frac{L|S||A|\log T}{\gapmin} \leq U$ finishes the proof.
\end{proof}

\subsection{Different tuning for the stochastic case}
\label{app:appendix_parameter_choices}

Here, we consider the case with stochastic losses (or more generally the case where Condition~\eqref{eq:self_bounding_constraint} holds with $C=0$), and point out what the best bound one can get by tuning $\alpha$ and $\gamma$ optimally.
For simplicity, we consider the worst case when $\gap(s,a)=\gapmin$ holds for all state-action pairs $(s,a)$ with $a\neq \pi^\star(s)$.
Repeating the same argument in the proof of \pref{col:stochastic_bound} with the general bound from \pref{thm:main_general}, one can verify that the final bound is
\[
\begin{split}
\Reg_T & \leq \order\rbr{\sbr{\frac{1+\alpha}{\gamma} + \gamma L^2\rbr{\sqrt{L}+\frac{1}{\alpha L}}}^2 \frac{|S||A|\log T}{\gapmin} + \rbr{\frac{\rbr{1+\alpha|A|}\sqrt{|S|L}}{\gamma}}^2 \frac{\log T}{\gapmin}} \\
& + \order\rbr{L|S||A|\log T}.
\end{split}
\] 
Picking the optimal parameters leads to
\[
\Reg_T  \leq  \order\rbr{|S|\sqrt{|A|}L^{\nicefrac{3}{2}}\rbr{\sqrt{|A|}L+L^{\nicefrac{3}{2}}+|A|+A^{\nicefrac{3}{4}}\sqrt{L}}\frac{\log T}{\gapmin} + L|S||A|\log T}.
\]
This is better than the bound stated in \pref{col:stochastic_bound}, and could even be better than the bound
$\order\rbr{\frac{L^3|S||A|\log T}{\gapmin}}$ achieved by StrongEuler~\citep{simc2019} (although they consider a harder setting where the transition function is unknown). 
However, this set of parameters leads to a sub-optimal bound for the adversarial case unfortunately.


\subsection{The Hessian of $\phi_H$}\label{app:Hessian}

We calculate the Hessian of our hybrid regularizer $\phi_H$ in this section, which is important for the analysis in \pref{app:appendix_proof_ftrl} and \pref{app:appendix_proof_stability}.
As mentioned in \pref{sec:analysis} (first step), one important trick we do is to use a different but equivalent definition of $q(s)$.
Specifically, we analyze the following regularizer:
\begin{equation}\label{eq:hybrid_alternative}
\begin{split}
\phi_{H}(q) &= - \sum_{s\neq s_L,a\in A} \rbr{\sqrt{ q(s,a) } + \alpha \sqrt{ q(s) - q(s,a)}}, \\
&\text{where}\;\; q(s) = 
\begin{cases}
\sum_{s'\in S_{k(s)-1}}\sum_{a'\in A}q(s',a')P(s|s',a'), &\text{if $k(s) \neq 0$,}\\
1, &\text{else.}
\end{cases}
\end{split}
\end{equation}
We emphasize again that, within the feasible set $\Omega$, this definition is exactly equivalent to \pref{eq:hybrid}, and thus we are not changing the algorithm at all.

\begin{lemma} \label{lem:hessian_derivaties}
The Hessian of the regularizer $\phi_{H}(q)$ defined in \pref{eq:hybrid_alternative} is specified by the following:
\begin{itemize}
\item
for any $k = 1, \ldots, L-1$, $s' \in S_{k-1}$, $s \in S_{k}$, and $a, a'\in A$,
\begin{equation}\label{eq:derivative_1}
\frac{\partial^2 \phi_{H}}{\partial q(s',a') \partial q(s,a)} = \frac{-\alpha P(s|s',a')}{4\rbr{q(s) - q(s,a)}^{\nicefrac{3}{2}}};
\end{equation}

\item
for any $k = 1, \ldots, L-1$, $s, s' \in S_{k-1}$, and $(s,a) \neq (s',a')$,
\begin{equation}\label{eq:derivative_2}
\frac{\partial^2 \phi_{H}}{\partial q(s',a') \partial q(s,a)} = \sum_{s''\in S_k}\sum_{a''\in A}\frac{\alpha P(s''|s,a)P(s''|s',a')}{4(q(s'')-q(s'',a''))^{\nicefrac{3}{2}}};
\end{equation}

\item 
for any $k = 1, \ldots, L$, $s \in S_{k-1}$, and $a \in A$,
\begin{equation}\label{eq:derivative_3}
\begin{split}
\frac{\partial^2 \phi_{H}}{\partial q(s,a)^2} &= 
 \frac{1}{4q(s,a)^{\nicefrac{3}{2}}} + \frac{\alpha}{4(q(s) - q(s,a))^{\nicefrac{3}{2}}}
+ \sum_{s'\in S_{k}, s'\neq s_L}\sum_{a'\in A}\frac{\alpha P(s'|s,a)^2}{4(q(s')-q(s',a'))^{\nicefrac{3}{2}}};
\end{split}
\end{equation}

\item all other entries of the Hessian are $0$.
\end{itemize}
Moreover, for any $w : (S\setminus\{s_L\}) \times A \rightarrow \fR$, the quadratic form $w^\top \nabla^2 \phi_H(q) w$ can be written as
\begin{equation}\label{eq:quadratic_form}
\begin{split}
&\frac{1}{4}\sum_{s\neq s_L}\sum_a \rbr{\frac{w(s,a)^2}{q(s,a)^{\nicefrac{3}{2}}} +
\frac{\alpha (h(s) - w(s,a))^2}{(q(s)-q(s,a))^{\nicefrac{3}{2}}}}  \\
&\quad\text{where}\;\; h(s) = 
\begin{cases}
\sum_{s' \in S_{k(s)-1}}\sum_{a'\in A} P(s |s',a')w(s',a'), &\text{if $k(s) > 0$,} \\
0, &\text{else.}
\end{cases}
\end{split}
\end{equation}
Consequently, $\phi_H$ is convex in $q$.
\end{lemma}

\begin{proof}	
Fix a state-action pair $(s,a)$.  By direct calculation, we show that the terms in $\phi_{H}(q)$ containing the variable $q(s,a)$ are
\begin{equation}
\label{eq:related_terms}
-\rbr{\sqrt{q(s,a)}+\alpha\sqrt{q(s)-q(s,a)}} - \alpha\sum_{s''\in S_{k(s)+1}}\sum_{a''\in A}\sqrt{q(s'') - q(s'',a'')}
\end{equation}
where the last term is zero when $s = s_L$. From \pref{eq:related_terms}, we can infer that the second-order partial derivatives of $q(s,a)$ and $q(s',a')$ are non-zero if and only if $|k(s) - k(s') | = 1$. 

We first verify \pref{eq:derivative_1}, where $k(s') = k(s) - 1$ and the derivatives are from $-\alpha\sqrt{q(s) - q(s,a)}$. Direct calculations shows 
\begin{align*}
\frac{\partial^2 \phi_{H}}{\partial q(s',a') \partial q(s,a)} & =\frac{\partial^2 }{\partial q(s',a') \partial q(s,a)}\rbr{-\alpha\sqrt{q(s) - q(s,a)} }  \\ 
& = \frac{\partial }{\partial q(s,a)}\rbr{ \frac{-\alpha P(s|s',a')}{2\sqrt{q(s) - q(s,a) }} } \\
& =  \frac{-\alpha P(s|s',a')}{4\rbr{q(s) - q(s,a)}^{\nicefrac{3}{2}}}
\end{align*}
where the second step follows because the term $q(s',a')P(s|s',a')$ belongs to $q(s)$ by \pref{eq:hybrid_alternative}.

For the second case (\pref{eq:derivative_2}) where $k(s) = k(s')$ and $(s,a)\neq (s',a')$ , we have that 
\begin{align*}
\frac{\partial^2 \phi_{H}}{\partial q(s',a') \partial q(s,a)} & 
=\frac{\partial^2 }{\partial q(s',a') \partial q(s,a)}\rbr{- \alpha\sum_{s''\in S_{k(s)+1}}\sum_{a''\in A}\sqrt{q(s'') - q(s'',a'')} }  \\ 
& 
=\sum_{s''\in S_{k(s)+1}}\sum_{a''\in A}\frac{\partial^2 }{\partial q(s',a') \partial q(s,a)}\rbr{- \alpha\sqrt{q(s'') - q(s'',a'')} }  \\ 
& =  \sum_{s''\in S_{k(s)+1}}\sum_{a''\in A}\frac{\partial }{\partial q(s,a)}\rbr{ \frac{-\alpha P(s''|s',a')}{2\sqrt{q(s'') - q(s'',a'')}} } \\
& = \sum_{s''\in S_k}\sum_{a''\in A}\frac{\alpha P(s''|s,a)P(s''|s',a')}{4(q(s'')-q(s'',a''))^{\nicefrac{3}{2}}}
\end{align*}
where the third step follows from the previous calculation. 

Finally, we finish the first part of the proof by verifying \pref{eq:derivative_3} that  
\begin{align*}
\frac{\partial^2 \phi_{H}}{\partial q(s,a)^2} &=  -\frac{\partial^2 }{\partial q(s,a)^2}\rbr{\sqrt{q(s,a)}+\alpha\sqrt{q(s)-q(s,a)}} \\ & \quad  - \sum_{s''\in S_{k(s)+1}}\sum_{a''\in A} \frac{\partial^2 }{\partial q(s,a)^2} \rbr{\alpha \sqrt{q(s'') - q(s'',a'')}} \\ &=
\frac{1}{4q(s,a)^{\nicefrac{3}{2}}} + \frac{\alpha}{4(q(s) - q(s,a))^{\nicefrac{3}{2}}}
+ \sum_{s'\in S_{k}, s'\neq s_L}\sum_{a'\in A}\frac{\alpha P(s'|s,a)^2}{4(q(s')-q(s',a'))^{\nicefrac{3}{2}}}.
\end{align*}

Then, we are ready to prove \pref{eq:quadratic_form}.
Indeed, we have
\begin{align*}
&w^\top \nabla^2 \phi_H(q) w
= \sum_{s\neq s_L}\sum_a \rbr{\frac{w(s,a)^2}{4q(s,a)^{\nicefrac{3}{2}}} + \frac{\alpha w(s,a)^2}{4(q(s)-q(s,a))^{\nicefrac{3}{2}}}} \\
&\quad + \sum_{0<k<L}\sum_{s \in S_{k-1}}\sum_{a} w(s,a)^2 \sum_{s''\in S_{k}}\sum_{a''\in A}\frac{\alpha P(s''|s,a)^2}{4(q(s'')-q(s'',a''))^{\nicefrac{3}{2}}} \\
&\quad + \sum_{0<k<L}\sum_{s, s' \in S_{k-1}}\sum_{a,a'\in A: (s,a)\neq (s',a')} w(s,a) w(s',a') \sum_{s''\in S_k}\sum_{a''\in A}\frac{\alpha P(s''|s,a)P(s''|s',a')}{4(q(s'')-q(s'',a''))^{\nicefrac{3}{2}}} \\
&\quad - 2\sum_{0<k<L}\sum_{s\in S_k}\sum_{s' \in S_{k-1}}\sum_{a, a'} \frac{\alpha P(s|s',a')w(s,a) w(s',a') }{4\rbr{q(s) - q(s,a)}^{\nicefrac{3}{2}}} \\
&= \sum_{s\neq s_L}\sum_a \rbr{\frac{w(s,a)^2}{4q(s,a)^{\nicefrac{3}{2}}} + \frac{\alpha w(s,a)^2}{4(q(s)-q(s,a))^{\nicefrac{3}{2}}}} \\
&\quad + \sum_{0<k<L} \sum_{s''\in S_{k}}\sum_{a''\in A}\frac{\alpha }{4(q(s'')-q(s'',a''))^{\nicefrac{3}{2}}}
\rbr{\sum_{s \in S_{k-1}}\sum_{a} w(s,a)P(s''|s,a)}^2\\
&\quad - 2\sum_{0<k<L}\sum_{s\in S_k}\sum_a \frac{\alpha w(s,a) \sum_{s' \in S_{k-1}}\sum_{a'}P(s|s',a') w(s',a') }{4\rbr{q(s) - q(s,a)}^{\nicefrac{3}{2}}} \\
&= \sum_{s\neq s_L}\sum_a \rbr{\frac{w(s,a)^2}{4q(s,a)^{\nicefrac{3}{2}}} + \frac{\alpha w(s,a)^2}{4(q(s)-q(s,a))^{\nicefrac{3}{2}}}} \\
&\quad + \sum_{0<k<L}\sum_{s\in S_{k}}\sum_{a\in A}\frac{\alpha h(s)^2 }{4(q(s)-q(s,a))^{\nicefrac{3}{2}}} 
- 2\sum_{0<k<L}\sum_{s\in S_k}\sum_{a} \frac{\alpha w(s,a) h(s) }{4\rbr{q(s) - q(s,a)}^{\nicefrac{3}{2}}} \\
&= \frac{1}{4}\sum_{s\neq s_L}\sum_a \rbr{\frac{w(s,a)^2}{q(s,a)^{\nicefrac{3}{2}}} +
\frac{\alpha (h(s) - w(s,a))^2}{(q(s)-q(s,a))^{\nicefrac{3}{2}}}},
\end{align*}
finishing the proof.
\end{proof}

Note that the Hessian is clearly non-diagonal.
However, by using the alternative definition from \pref{eq:hybrid_alternative}, the Hessian has a layered structure which allows an induction-based analysis as we will show.

\section{Proof of \pref{lem:key_lemma_ftrl}}
\label{app:appendix_proof_ftrl}

In this section, we provide the proof for \pref{lem:key_lemma_ftrl}.
First, we introduce the following notations.
\begin{equation}\label{eq:tilde_q_def}
\begin{split}
F_{t}(q) =  \inner{q, \widehat{L}_{t-1}}  + \psi_{t}(q) \;,& \;\quad  G_{t} =  \inner{q_t, \widehat{L}_{t}}  + \psi_{t}(q), \\ 
q_{t} = \argmin_{q\in \Omega} F_t(q) \;,&\;\quad \widetilde{q}_t = \argmin_{q \in \Omega} G_t(q). 
\end{split}
\end{equation}
Note that the definition of $q_t$ is consistent with \pref{alg:main_alg}.
With these notations, we decompose the regret $\E\sbr{\sum_{t=1}^{T} \inner{q_t - u, \hatl_t}}$ against any occupancy measure $u \in \Omega$ into a stability term and a mixed penalty term by adding and subtracting $F_t(q_t)$ and $G_t(\tilde{q}_t)$:
\begin{equation}\label{eq:simple_decomposition}
\begin{split}
 = \underbrace{\E\sbr{\sum_{t=1}^{T} \rbr{\inner{q_t,\hatl_t} + F_t(q_t) - G_t(\tilde{q}_t)}}}_{\text{stability}}+ \underbrace{\E\sbr{\sum_{t=1}^{T} \rbr{ G_t(\tilde{q}_t) - F_t(q_t) - \inner{u, \hatl_t}}}}_{\text{mixed penalty}}.
\end{split}
\end{equation}

%

The rest of the section is organized as follows:
\begin{enumerate}
	\item With some auxiliary lemmas (\pref{lem:convex_combine} and \pref{lem:close_hessian}), we prove in \pref{lem:smooth_update} that the update of the algorithm is smooth in the sense that $q_t$ and $\tilde{q}_t$ are close.
	\item With the smoothness guarantee, we bound the stability term by the quadratic norm of $\hatl_t$ with respect to the Hessian matrix in \pref{lem:bound_stability}, using similar techniques from~\citep{lee2020closer}.
	\item By standard analysis, we control the mixed penalty term in \pref{lem:bound_penalty}.
	\item Picking a proper $u$ that is closed to $\opt$ and specified in \pref{lem:proper_comparator}, we finish the proof of \pref{lem:key_lemma_ftrl} in the end of this section. 
\end{enumerate}

We use the notation $M_1 \preceq M_2$ for two  matrices $M_1$ and $M_2$ to denote the fact that $M_2 - M_1$ is positive semi-definite.

\begin{lemma}[Convexity of $\Omega$]
	\label{lem:convex_combine}
	The set of valid occupancy measures $\Omega$ is convex.
\end{lemma}
\begin{proof}
	For any $u, v\in \Omega$ and $\lambda \in [0,1]$,
	it suffices to verify that $p=\lambda u + (1-\lambda)v$ satisfies the two constraints described in \pref{sec:probform}.
	For the first one: we have for any $k = 0, \ldots, L-1$, 
	\[
	\begin{split}
	\sum_{s\in S_k}\sum_{a\in A}p(s,a) & = \lambda \sum_{s\in S_k}\sum_{a\in A}u(s,a) +  \rbr{1-\lambda} \sum_{s\in S_k}\sum_{a\in A}v(s,a)  \\
	& = \lambda + ( 1 - \lambda ) = 1.
	\end{split}
	\]
	For the second one (\pref{eq:occupancy_measure_cond_2}),
	 we have for any $k = 1, \ldots, L-1$ and every state $s\in S_k$,
	\[
	\begin{split}
	& \sum_{s' \in S_{k-1}, a'} P(s|s',a')p(s',a') \\
	& = \sum_{s'\in S_{k-1},a'} P(s|s',a')\rbr{\lambda u(s',a') + (1-\lambda) v(s',a')} \\
	& = \lambda\sum_{s'\in S_{k-1},a'}P(s|s',a')u(s',a')+ (1-\lambda)\sum_{s'\in S_{k-1},a'}P(s|s',a')v(s',a') \\
	& = \lambda \sum_{a} u(s,a) + (1-\lambda)\sum_{a} v(s,a) \\ 
	& = \sum_{a}  \rbr{\lambda u(s,a) + (1-\lambda) v(s,a)} \\ 
	& = \sum_{a} p(s,a).
	\end{split}
	\]
	This finishes the proof.
\end{proof}

\begin{lemma}
	For any occupancy measures $q$ and $p$ from $\Omega$ satisfying $\frac{1}{2}q(s,a) \leq p(s,a) $ for all state-action pair $(s,a)$, we have $\frac{1}{4} \nabla^2  \psi_t(p) \preceq \nabla^2 \psi_t(q)$. 
	\label{lem:close_hessian}
\end{lemma}
\begin{proof} Let $M_1 =  \nabla^2 \psi_t(p)$ and $M_2 =  \nabla^2  \psi_t(q)$. 
By \pref{eq:quadratic_form} in \pref{lem:hessian_derivaties}, we have for any $w$, 
	\[
	\begin{split}
	w^\top M_2 w & = \frac{1}{4\eta_t}\sum_{s\neq s_L,a}\cbr{\frac{w(s,a)^2}{q(s,a)^{\nicefrac{3}{2}}}
		+\frac{\alpha \rbr{h(s) - w(s,a)}^2}{(q(s)-q(s,a))^{\nicefrac{3}{2}}}} + \paralog \sum_{s\neq s_L,a} \frac{w(s,a)^2}{q(s,a)^2}.
	\end{split}
	\]	
	According to the condition of the lemma and the fact $q(s)-q(s,a) = \sum_{b\neq a}q(s,b)$ and $p(s)-p(s,a) = \sum_{b\neq a}p(s,b)$, we have $ \frac{1}{2}\rbr{q(s)-q(s,a)} \leq p(s)-p(s,a)$, and thus 
	\[
	\begin{split}
	w^\top M_2 w&\geq \frac{1}{4\eta_t}\sum_{s\neq s_L,a}\cbr{\frac{w(s,a)^2}{(2p(s,a))^{\nicefrac{3}{2}}}
		+\frac{\alpha \rbr{h(s)  - w(s,a)}^2}{(2(p(s)-p(s,a)))^{\nicefrac{3}{2}}}} + \paralog\sum_{s,a} \frac{w(s,a)^2}{(2p(s,a))^2} \\
	& \geq \frac{1}{16\eta_t}\sum_{s\neq s_L,a}\cbr{\frac{w(s,a)^2}{p(s,a)^{\nicefrac{3}{2}}}
		+\frac{\alpha \rbr{h(s)  - w(s,a)}^2}{(p(s)-p(s,a))^{\nicefrac{3}{2}}}} + \frac{\paralog}{4}\sum_{s,a} \frac{w(s,a)^2}{p(s,a)^2} \\
	& =\frac{1}{4} w^\top M_1 w.
	\end{split}
	\]
	This finishes the proof.
\end{proof}

With the help of \pref{lem:convex_combine} and \pref{lem:close_hessian}, we now prove that  $q_t$ to $\widetilde{q}_t$ are close.

\begin{lemma}
	With $\paralog = 64L$, 
	we have $\frac{1}{2} q_t(s,a) \leq \tilde{q}_{t}(s,a) \leq 2 q_t(s,a)$ for all state-action pairs $(s,a)$ (recall the definitions in \pref{eq:tilde_q_def}). 
	\label{lem:smooth_update}
\end{lemma}
\begin{proof}
	For simplicity, we denote $H$ as the Hessian $\nabla^2 \psi_t(q_t)$, and $H_{L}$ as the Hessian $\nabla^2 \phi_{L}(q_t)$ which is a diagonal matrix with $\frac{\paralog}{q_t(s,a)^2}$ on the diagonal. Recalling that $\psi_t = \eta_t^{-1}\phi_{H} + \phi_{L}$,  by the convexity of $\phi_{H}$ (\pref{lem:hessian_derivaties}), we have $H_{L} \preceq H$.
	Our goal is to prove $\norm{\tilde{q}_t-q_t}_{H} \leq 1$.
	This is enough to prove the lemma statement because 
	\begin{equation}\label{eq:earlier_argument}
	\begin{split}
	1 \geq \norm{\tilde{q}_t-q_t}_{H}  \geq \norm{\tilde{q}_t-q_t}_{H_L}  = \paralog \sum_{s,a}\frac{\rbr{\tilde{q}_t(s,a) - q_t(s,a)}^2 }{q_t(s,a)^2},
	\end{split}
	\end{equation}
	and since $\paralog \geq 9$, we have $\rbr{\tilde{q}_t(s,a) - q_t(s,a)}^2 \leq \frac{q_t(s,a)^2}{\paralog}  \leq  \rbr{\frac{q_t(s,a)}{3}}^2$, which indicates $\frac{1}{2} q_t(s,a) \leq  \tilde{q}_t(s,a) \leq  2 q_t(s,a)$. 
	
	To prove $\norm{\tilde{q}_t-q_t}_{H} \leq 1$, it suffices to show that for any occupancy measure $q'$  that satisfies $\norm{q' - q_t}_{H} = 1$, we have $G_t(q') \geq G_t(q_t)$, because this implies that, as the minimizer of the convex function $G_t$, $\tilde{q}_t$ must be within the convex set $\{q: \norm{q - q_t}_{H} \leq 1\}$.
	To this end, we bound $G_t(q')$ as follows
	\[
	\begin{split}
	G_t(q') & = G_t(q_t) + \nabla G_t(q_t)^\top(q'-q_t) + \tfrac{1}{2}\norm{q'-q_t}_{\nabla^2 \psi_t(\xi)}^2 \\
	& = G_t(q_t) + \nabla F_t(q_t)^\top(q'-q_t) + \hatl_t^\top(q'-q_t) + \tfrac{1}{2}\norm{q'-q_t}_{\nabla^2 \psi_t(\xi)}^2\\
	& \geq G_t(q_t) - \norm{\hatl_t}_{H^{-1}} \norm{q'-q_t}_{H} +\tfrac{1}{2}\norm{q'-q_t}_{\nabla^2 \psi_t(\xi)}^2 \\
	& = G_t(q_t) - \norm{\hatl_t}_{H^{-1}}+ \tfrac{1}{2}\norm{q'-q_t}_{\nabla^2 \psi_t(\xi)}^2
	\end{split}
	\]
	where in the first step we use Taylor's expansion with $\xi$ being a point between $q_t$ and $q'$,
	in the second step we use the definition of $G_t$ and $F_t$ (in \pref{eq:tilde_q_def}),
	and in the third step we use \Holder's inequality and the first order optimality condition $\inner{\nabla F_t(q_t), q' - q_t} \geq 0$. 
	
	Repeating the earlier argument in \pref{eq:earlier_argument}, $\norm{q' - q_t}_{H} = 1$ also implies $\frac{1}{2} q_t(s,a) \leq  q'(s,a) \leq  2 q_t(s,a)$ and thus $\frac{1}{2} q_t(s,a) \leq  \xi(s,a) \leq  2 q_t(s,a)$.
	Since $\xi$ is in $\Omega$ by \pref{lem:convex_combine}, we continue to bound the last expression using \pref{lem:close_hessian}:
	\begin{align}
	& G_t(q_t) - \norm{\hatl_t}_{H^{-1}}+ \tfrac{1}{2}\norm{q'-q_t}_{\nabla^2 \psi_t(\xi)}^2 \notag  \\ 
	&\geq  G_t(q_t) - \norm{\hatl_t}_{H^{-1}}+ \tfrac{1}{8}\norm{q'-q}_{H}^2 \notag \\  
	& = G_t(q_t) - \norm{\hatl_t}_{H^{-1}} + \tfrac{1}{8}.
	\label{eq:convex_eq}
	\end{align}
	Finally, we bound the term $\norm{\hatl_t}_{H^{-1}}$. By the definition of $\hatl_t$, with shorthand $\mathbb{I}(s,a) \triangleq \Ind{s_{k(s)}=s,a_{k(s)}= a}$ we show that  
	\begin{align*}
	\norm{\hatl_t}^2_{H^{-1}} & \leq \norm{\hatl_t}^2_{H_L^{-1}}  \tag{$H_{L} \preceq H$} \\
	& = \sum_{s,a} \frac{\mathbb{I}(s,a)\ell_t(s,a)^2}{q_t(s,a)^2} \frac{q_t(s,a)^2}{\beta} \\
	& \leq \sum_{s,a} \frac{\mathbb{I}\rbr{s,a}}{\paralog}  = \frac{L}{\paralog} = \frac{1}{64}.
	\end{align*}
	Plugging it into \pref{eq:convex_eq}, we conclude that $G_t(q') \geq G_t(q_t)$, which finishes the proof.
\end{proof}


We are now ready to bound the first term in \pref{eq:simple_decomposition}.
\begin{lemma} With $\paralog=64L$, we have
	\[
	\sum_{t=1}^{T} \rbr{\inner{q_t,\hatl_t} + F_t(q_t) - G_t(\tilde{q}_t)} \leq 8\sum_{t=1}^{T}\norm{\hatl_t}^2_{\nabla^{-2}\psi_t(q_t)}.
	\]
	\label{lem:bound_stability}
\end{lemma}
\begin{proof} We first lower bound the term $\inner{q_t,\hatl_t} + F_t(q_t) - G_t(\tilde{q}_t)$ as
	\[
	\begin{split}
	& \inner{q_t,\hatl_t} + F_t(q_t) - G_t(\tilde{q}_t) \\
	& = \inner{q_t, \hatl_t + \widehat{L}_{t-1}} + \psi_t(q_t) - G_t(\tilde{q}_t) \\ 
	& =  G_t(q_t) - G_t(\tilde{q}_t) \\
	& = \inner{\nabla G_t(\tilde{q}_t),q _t - \tilde{q}_t } + \tfrac{1}{2} \norm{q_t - \tilde{q}_t}^2_{\nabla^2 \psi_t(\xi)} \\
	& \geq  \tfrac{1}{2} \norm{q_t - \tilde{q}_t}^2_{\nabla^2 \psi_t(\xi)},
	\end{split}
	\]
	where in the second to last step we apply Taylor's expansion with $\xi$ being a point between $q_t$ and $\tilde{q}_t$, and in the last step we use the first order optimality condition of $\tilde{q}_t$.
	
	On the other hand, we can upper bound this term as
	\[
	\begin{split}
	& \inner{q_t,\hatl_t} + F_t(q_t) - G_t(\tilde{q}_t) \\
	& = \inner{q_t - \tilde{q}_t ,\hatl_t} + F_t(q_t) - F_t(\tilde{q}_t) \\ 
	& \leq \inner{q_t - \tilde{q}_t ,\hatl_t}  \\ 
	& \leq \norm{q_t - \tilde{q}_t}_{\nabla^2 \psi_t(\xi)} \norm{\hatl_t}_{\nabla^{-2} \psi_t(\xi)},
	\end{split}
	\]
	where the second step is by the optimality of $q_t$ and the last step is by \Holder's inequality.
	Combining the lower and upper bounds we arrive at:
	\[
	\inner{q_t,\hatl_t} + F_t(q_t) - G_t(\tilde{q}_t) \leq 2 \norm{\hatl_t}^2_{\nabla^{-2} \psi_t(\xi)}.
	\]
	Moreover, by \pref{lem:smooth_update}, we know that $\tilde{q}_t$ satisfies $\frac{1}{2} q_t(s,a) \leq \tilde{q}_t(s,a) \leq 2q_t(s,a)$ for all state-action pairs. Since $\xi$ is a middle point between $q_t$ and $\tilde{q}_t$, it satisfies $\frac{1}{2} q_t(s,a) \leq \xi(s,a) \leq 2q_t(s,a)$ for all state-action pairs as well. 
	According to \pref{lem:close_hessian}, we then have $\norm{\hatl_t}^2_{\nabla^{-2} \psi_t(\xi)} \leq 4\norm{\hatl_t}^2_{\nabla^{-2} \psi_t(q_t)}$,
	which completes the proof.
\end{proof}

Next, we bound the second term in \pref{eq:simple_decomposition}.
\begin{lemma}
	For any $u\in \Omega$, we have (with $1/\eta_0 \triangleq 0$)
	\[
	\begin{split}
	\sum_{t=1}^{T} \rbr{ G_t(\tilde{q}_t) - F_t(q_t) - \inner{u, \hatl_t} }  
 = \phi_{L}(u) - \phi_{L}(q_1) + \sum_{t=1}^{T}\rbr{\frac{1}{\eta_t} - \frac{1}{\eta_{t-1}} }\rbr{\phi_{H}(u) - \phi_{H}(q_t)}.
	\end{split}
	\]
	\label{lem:bound_penalty}
\end{lemma}
\begin{proof} Due to the optimality of $\tilde{q}_t$, we have $G_t(\tilde{q}_t) \leq G_t(q_{t+1})$ and also $G_T(\tilde{q}_T) \leq G_T(u)$.  With the help of these inequalities, we proceed as
	\[
	\begin{split}
	& \sum_{t=1}^{T} \rbr{ G_t(\tilde{q}_t) - F_t(q_t) - \inner{u, \hatl_t}} \\ 
	 &\leq \sum_{t=1}^{T} \rbr{ G_t(q_{t+1}) - F_t(q_t) - \inner{u, \hatl_t}} \\ 
	&=  - F_1(q_1) + \sum_{t=2}^{T} \rbr{ G_{t-1}(q_{t}) - F_t(q_t)} + G_T(u) - \inner{u, \hatL_T} \\
	&= - \psi_1(q_1) - \sum_{t=2}^{T}\rbr{\frac{1}{\eta_t} - \frac{1}{\eta_{t-1}} }\phi_{H}(q_t) + \psi_T(u) \\
   &= - \phi_{H}(q_1) -  \phi_{L}(q_1) - \sum_{t=2}^{T}\rbr{\frac{1}{\eta_t} - \frac{1}{\eta_{t-1}} }\phi_{H}(q_t) + \frac{\phi_{H}(u)}{\eta_T} + \phi_{L}(u) \\
	&= \phi_{L}(u)  - \phi_{L}(q_1) + \sum_{t=1}^{T} \rbr{\frac{1}{\eta_t} - \frac{1}{\eta_{t-1}} }\rbr{\phi_{H}(u) - \phi_{H}(q_t)}, 
	\end{split}
	\]
	finishing the proof.
\end{proof}


While it is tempting to set $u = \opt$ to obtain a regret bound, note that $\phi_{L}(\opt)$ can potentially grow to infinity. To this end, we will set $u$ as some point close enough to $\opt$, specified in the following lemma.

\begin{lemma}\label{lem:proper_comparator} 
Suppose $\paralog = 64L$. Let $v = \rbr{1 - \frac{1}{T}}\opt +  \frac{q_{1}}{T}$, where $\opt$ is the optimal occupancy measure and $q_1 =  \argmin_{q\in\Omega}  \psi_1(q)$ is the initial occupancy measure.
Then $v$ satisfies the following:
\begin{itemize}
	\item $\E\sbr{\sum_{t=1}^{T}\inner{v-\opt, \hatl_t}} \leq 2L$,
	\item $\phi_{L}(v)-\phi_{L}(q_1) \leq 64L|S||A|\log T$,
	\item $\phi_{H}(v) - \phi_{H}(\opt) \leq \frac{\rbr{1+\alpha} |S||A|}{T}.$
\end{itemize}	
\end{lemma}

\begin{proof}

The first statement is by direct calculation:
\[
\begin{split}
\E\sbr{\sum_{t=1}^{T}\inner{v-\opt, \hatl_t}} & = \frac{1}{T} \E\sbr{\sum_{t=1}^{T}\inner{q_1-\opt, \hatl_t}} \\
& = \frac{1}{T}\inner{q_1-\opt,\sum_{t=1}^{T}\ell_t} \\
& \leq \frac{1}{T} \norm{q_1-\opt}_1 \norm{\sum_{t=1}^{T}\ell_t}_{\infty} \leq \frac{2LT}{T} = 2L .
\end{split}
\]

The second statement directly uses the definition of $\phi_L$:
\[
\phi_{L}(v)-\phi_{L}(q_1) = 64L\sum_{s,a}\log\rbr{\frac{q_1(s,a)}{v(s,a)} } \leq 64L|S||A|\log(T).
\]

Finally, we verify the last statement: 
	\[
	\begin{split}
	& \phi_{H}(v) - \phi_{H}(\opt) \\
	= & \sum_{s,a} \rbr{ \sqrt{\opt(s,a)}  + \alpha \sqrt{\opt(s)-\opt(s,a)} - \sqrt{v(s,a)} - \alpha \sqrt{v(s)-v(s,a)} } \\
	=  & \sum_{s,a} \rbr{  \sqrt{\opt(s,a)}   - \sqrt{\frac{T-1}{T}\opt(s,a) + \frac{q_1(s,a) }{T} } } \\
	& +  \sum_{s,a} \alpha  \rbr{  \sqrt{\opt(s)-\opt(s,a)}- \sqrt{\frac{T-1}{T}(\opt(s)-\opt(s,a)) + \frac{q_1(s)-q_1(s,a) }{T} } } \\
	\leq & \rbr{1 - \sqrt{\frac{T-1}{T}} }\sum_{s,a} \rbr{ \sqrt{\opt(s,a)}  +  \alpha  \sqrt{\opt(s)-\opt(s,a)} } \\ 
	= & \frac{\rbr{\sqrt{T}-\sqrt{T-1}}\rbr{\sqrt{T} +\sqrt{T-1}}}{\sqrt{T}\rbr{\sqrt{T} +\sqrt{T-1}}} \rbr{\rbr{1+\alpha} |S||A|} \\
	\leq & \frac{\rbr{1+\alpha} |S||A|}{T}.
	\end{split}
	\]
	
\end{proof}

Finally, we are ready to prove \pref{lem:key_lemma_ftrl}.
\begin{proof}[Proof of \pref{lem:key_lemma_ftrl}]
\label{prf:proof_ftrl}

By combining the \pref{lem:bound_penalty} and~\ref{lem:bound_stability} and taking expectation on both sides, we have that $\E\sbr{\sum_{t=1}^{T} \inner{q_t - u, \hatl_t}}$ is bounded by
\[
\phi_{L}(u)  - \phi_{L}(q_1) + \sum_{t=1}^{T} \rbr{\frac{1}{\eta_t} - \frac{1}{\eta_{t-1}} }\E\sbr{\phi_{H}(u) - \phi_{H}(q_t)} + 	8\sum_{t=1}^{T}\E\sbr{\norm{\hatl_t}^2_{\nabla^{-2}\psi_t(q_t)}}.
\]

Picking the intermediate occupancy measure $v= \rbr{1-\frac{1}{T}}\opt + \frac{1}{T}q_1$, by \pref{lem:proper_comparator}, we have 
\[
\begin{split}
\Reg_T & = \E\sbr{\sum_{t=1}^{T} \inner{q_t - \opt, \hatl_t}} \\
& = \E\sbr{\sum_{t=1}^{T} \inner{q_t - v, \hatl_t}} + \E\sbr{\sum_{t=1}^{T} \inner{v-\opt, \hatl_t}} \\
& \leq 2L + 64L|S||A|\log T + \sum_{t=1}^{T} \rbr{\frac{1}{\eta_t} - \frac{1}{\eta_{t-1}} }\E\sbr{\phi_{H}(v) - \phi_{H}(q_t)} + 8\sum_{t=1}^{T}\E\sbr{\norm{\hatl_t}^2_{\nabla^{-2}\psi_t(q_t)}} \\
& \leq  2L + 64L|S||A|\log T  +  8\sum_{t=1}^{T}\E\sbr{\norm{\hatl_t}^2_{\nabla^{-2}\psi_t(q_t)} } \\
& \; \; + \rbr{\frac{1}{\eta_T} - \frac{1}{\eta_{0}} }\E\sbr{\phi_{H}(v) - \phi_{H}(\opt)} +  \sum_{t=1}^{T} \rbr{\frac{1}{\eta_t} - \frac{1}{\eta_{t-1}} }\E\sbr{\phi_{H}(\opt) - \phi_{H}(q_t)} \\
& \leq \order\rbr{L|S||A|\log T} + \sum_{t=1}^{T} \rbr{\frac{1}{\eta_t} - \frac{1}{\eta_{t-1}} }\E\sbr{\phi_{H}(\opt) - \phi_{H}(q_t)} + 8\sum_{t=1}^{T}\E\sbr{\norm{\hatl_t}^2_{\nabla^{-2}\psi_t(q_t)}}
\end{split}
\]
where the last line follows from the fact $\frac{(1+\alpha)|S||A|}{\gamma\sqrt{T}} = \order\rbr{1}$. 
\end{proof}

\section{Proof of \pref{lem:key_lemma_bound_penalty}}
\label{app:appendix_proof_penalty}
Recall the notation defined in Section~\ref{sec:analysis}: 
\[
h_s(\pi)= \sum_{a\in A}\sqrt{\pi(a|s)}+\alpha\sqrt{1-\pi(a|s)}.
\]

Clearly, for any mapping (deterministic policy) $\pi : S \rightarrow A$, we have $h_s(\pi) = 1 + \alpha(|A|-1)$ for all state $s$. Therefore, we can decompose $\phi_{H}(\opt) - \phi_{H}(q_t)$ as
\[
\begin{split}
\phi_{H}(\opt) - \phi_{H}(q_t) & = \sum_{s \neq s_L}\sqrt{q_t(s)}h_s(\pi_t) - \sum_{s \neq s_L}\sqrt{\opt(s)}h_s(\optpi) \\
& = \sum_{s \neq s_L}\sqrt{q_t(s)}\rbr{h_s(\pi_t)-h_s(\optpi)} + \rbr{1 + \alpha(|A|-1)}\sum_{s \neq s_L}\rbr{\sqrt{q_t(s)} - \sqrt{\opt(s)}}. 
\end{split}
\]
In the rest of this section, we first bound the term $\sum_{s \neq s_L}\sqrt{q_t(s)}\rbr{h_s(\pi_t)-h_s(\optpi)}$ in \pref{lem:penalty_term_1}, and then show that the term $\sum_{s \neq s_L}\rbr{\sqrt{q_t(s)} - \sqrt{\opt(s)}}$ can be bounded as 
\[
\begin{split}
\sum_{s \neq s_L}\rbr{\sqrt{q_t(s)} - \sqrt{\opt(s)}} & = \sum_{s \neq s_L}\rbr{\sqrt{q_t(s)} - \sqrt{q^\pi(s)}} +  \sum_{s \neq s_L}\rbr{\sqrt{q^\pi(s)} - \sqrt{\opt(s)}} \\
& \leq \sqrt{|S|L}\rbr{\sqrt{\sum_{s \neq s_L}\sum_{a\neq \pi(s)}q_t(s,a)} + \sqrt{\sum_{s \neq s_L}\sum_{a\neq \pi(s)}\opt(s,a)}} \\
& \leq 2\sqrt{|S|L\sum_{s \neq s_L}\sum_{a\neq \pi(s)}q_t(s,a)+\opt(s,a)}
\end{split} 
\]
for any mapping $\pi$,
where the second line follows \pref{lem:sqrt_diff} and \pref{lem:sqrt_diff_2}, in which we apply the key induction \pref{lem:induction_state_reach_prob} based on the \textit{state reach probability} defined in \pref{lem:state_reach_probabilty}.

\begin{lemma} For any policy $\pi_1$ and mapping $\pi_2: S \rightarrow A$, we have
	\label{lem:penalty_term_1}
	\[
	\sum_{s\neq s_L}\sqrt{q_1(s)} \rbr{h_s(\pi_1) -  h_s(\pi_2)}  \leq  (1+\alpha) \sum_{s\neq s_L}\sum_{a \neq \pi(s)}\sqrt{q_1(s,a)},
	\]
	 where $\pi$ is any mapping from $S$ to $A$ and $q_1 = q^{\pi_1}$.
\end{lemma}

\begin{proof}
Direct calculation shows:
\begin{align*}
&h_s(\pi_1) -  h_s(\pi_2) \\
&= h_s(\pi_1)  - 1 - \alpha(|A|-1) \\
&= \sqrt{\pi_1(\pi(s)|s)}+\alpha\sqrt{1-\pi_1(\pi(s)|s)} - 1 + \sum_{a \neq \pi(s)} \rbr{\sqrt{\pi_1(a|s)}+\alpha\sqrt{1-\pi_1(a|s)} - \alpha} \\
&\leq \alpha\sqrt{1-\pi_1(\pi(s)|s)} +  \sum_{a \neq \pi(s)}\sqrt{\pi_1(a|s)} \\
&= \alpha\sqrt{\sum_{a \neq \pi(s)} \pi_1(a|s)} +  \sum_{a \neq \pi(s)}\sqrt{\pi_1(a|s)} \\
&\leq (\alpha + 1)  \sum_{a \neq \pi(s)}\sqrt{\pi_1(a|s)}.
\end{align*}
Multiplying both sides by $\sqrt{q_1(s)}$ and summing over $s$ prove the lemma.
%
%
%
%
\end{proof} 

\begin{lemma}
\label{lem:state_reach_probabilty} 
For any policy $\pi$, define its associated reachability probability $p^\pi: S \times A \times S \rightarrow [0,1]$ as
\begin{equation}\label{eq:reachability}
p^\pi(s'|s,a) = 
\begin{cases}
0, & \text{if } k(s') \leq k(s), \\
P(s'|s,a), & \text{if }  k(s') = k(s) + 1, \\ 
\sum_{s_{m}\in S_{k(s')-1}} p^\pi(s_{m}|s,a)\sum_{a}\pi(a|s_{m})P(s'|s_{m},a), & \text{if } k(s') > k(s) + 1,
\end{cases}
\end{equation}
which is simply the probability of reaching state $s'$ after taking action $a$ at state $s$ and then following policy $\pi$.
For any state-action pair $(s,a)$, policy $\pi$, and $k = 0,\ldots, L-1$, we have 
\begin{equation}
\label{eq:state_reach_prob_bound}
\sum_{s'\in S_k} p^\pi(s'|s,a) \leq 1,
\end{equation}
which implies $\sum_{s \neq s_L}p^\pi(s'|s,a)\leq L$.
\end{lemma}
\begin{proof} 
\pref{eq:state_reach_prob_bound} is clear just based on the definition of $p^\pi$.
We provide a proof by induction for completeness.
Clearly, it holds for all layer $l$ with $l\leq k(s)$ and also $l = k(s)+1$ where 
$
\sum_{s'\in S_l} p^\pi(s'|s,a) = \sum_{s'\in S_l}P(s'|s,a) = 1. 
$
Now assume that \pref{eq:state_reach_prob_bound} holds for some layer $l\geq k(s) + 1$. 
For layer $l+1$, we have 
\[
\begin{split}
& \sum_{s'\in S_{l+1}} p^\pi(s'|s,a) \\
& = \sum_{s'\in S_{l+1}} \rbr{\sum_{s_{m}\in S_l} p^\pi(s_{m}|s,a)\sum_{a}\pi(a|s_{m})P(s'|s_{m},a) } \\ 
& = \sum_{s_{m}\in S_l}p^\pi(s_{m}|s,a) \rbr{\sum_{s'\in S_{l+1}}\sum_{a}\pi(a|s_{m})P(s'|s_{m},a) } \\
& = \sum_{s_{m}\in S_l}p^\pi(s_{m}|s,a) \leq 1,
\end{split}
\]	
finishing the proof.
\end{proof}

Based on the concept of reachability probability, we prove the following key induction lemma. 

\begin{lemma} \label{lem:induction_state_reach_prob} If a policy $\pi$ and non-negative functions $f: S\times A\rightarrow \fR_+ \cup \{0\}$ and $g: S\times A \rightarrow \fR_+ \cup \{0\}$ satisfy
\[
f(s) \leq \sum_{s' \in S_{k(s)-1}} \sum_{a'\in A} g(s',a') P(s|s',a') + \sum_{s' \in S_{k(s)-1}}f(s')\cbr{\sum_{a' \in A}\pi(a'|s')P(s|s',a')}, \forall s \neq s_0 
\]
and $f(s_0) = 0$, then we have for all $s\neq s_0$,
\begin{equation}
\label{eq:induction_ineq}
f(s) \leq \sum_{k=0}^{k(s)-1}\sum_{s'\in S_k}\sum_{a'} g(s',a') p^\pi (s|s',a'),
\end{equation}
where $p^\pi$ is the  reachability probability defined in \pref{eq:reachability}.
\end{lemma}
\begin{proof} 
We prove the statement by induction. 
First, for $s \in S_1$, using the condition of the lemma we have
\[
\begin{split}
f(s) & \leq \sum_{s' \in S_{0}} \sum_{a'} g(s',a') P(s|s',a') + \sum_{s' \in S_{0}}f(s')\cbr{\sum_{a'}\pi(a'|s')P(s|s',a')} \\
& = \sum_{s' \in S_{0}} \sum_{a'} g(s',a') p^\pi(s|s',a') \\
\end{split}
\]
where the second line follows from the fact $f(s_0)=0$ and $p^\pi(s|s',a')=P(s|s',a')$ for $k(s)=k(s')+1$. 
This proves the base case.
	
Now assume that \pref{eq:induction_ineq} holds for all states in layer $l>0$. For $s \in S_{l+1}$, we have 
\[
\begin{split}
f(s) & \leq \sum_{s' \in S_{l}} \sum_{a'} g(s',a') P(s|s',a') + \sum_{s' \in S_l}f(s')\cbr{\sum_{a'}\pi(a'|s')P(s|s',a')} \\
& \leq \sum_{s' \in S_{l}} \sum_{a'} g(s',a') P(s|s',a') \\ 
& \; \; + \sum_{s' \in S_l} \rbr{\sum_{k=0}^{l-1}\sum_{s''\in S_k}\sum_{a''} g(s'',a'') p^\pi(s'|s'',a'')}\rbr{\sum_{a'}\pi(a'|s')P(s|s',a')} \\ 
& = \sum_{s' \in S_{l}} \sum_{a'} g(s',a') p^\pi(s|s',a') \\ 
& \; \; + \sum_{k=0}^{l-1}\sum_{s''\in S_k}\sum_{a''} g(s'',a'') \rbr{\sum_{s' \in S_l}p^\pi(s'|s'',a'') \sum_{a'}\pi(a'|s')P(s|s',a')} \\ 
&= \sum_{s' \in S_{l}} \sum_{a'} g(s',a') p^\pi(s|s',a') + \sum_{k=0}^{l-1}\sum_{s''\in S_k}\sum_{a''} g(s'',a'') p^\pi(s|s'',a'') \\ 
& = \sum_{k=0}^{l}\sum_{s'\in S_k} \sum_{a'} g(s',a') p^\pi(s|s',a'),
\end{split} 
\]
where the second step uses the induction hypothesis and the fourth step uses the definition of the reachability probability.
This finishes the induction.
\end{proof}

We now apply the induction lemma to prove the following two key lemmas.
\begin{lemma}
	\label{lem:sqrt_diff}
	For any policy $\pi_1$ and mapping $\pi_2: S \rightarrow A$, we have
	\begin{equation}
	\sum_{s\neq s_L }\sqrt{q_1(s)} - \sqrt{q_2(s)} 
	\leq \sqrt{|S|L} \sqrt{\sum_{s\neq s_L}\sum_{a \neq \pi_2(s)} q_1(s,a)}
	\end{equation}
	where we denote by $q_1 = q^\pi_1$ and $q_2 = q^\pi_2$ the  occupancy measures of $\pi_1$ and $\pi_2$ respectively. 
\end{lemma}
\begin{proof}
   We first bound $\sqrt{q_1(s)} - \sqrt{q_2(s)}$ by $\sqrt{\mathbb{I}_{s} \rbr{q_1(s)-q_2(s)}} $ where $\mathbb{I}_{s} \triangleq \Ind{ q_1(s) \geq q_2(s)}$. 
	Define $f(s) = \mathbb{I}_{s}\rbr{q_1(s)-q_2(s)}$ and $g(s,a) = \Ind{a\neq \pi_2(s)}q_1(s,a)$. 
	Our goal is to prove for any $s\neq s_0$:
	\begin{equation}
	\label{eq:key_ineq_sqrt_diff}
	f(s) \leq \sum_{s' \in S_{k(s)-1}}\sum_{a' \in A}g(s',a')P(s|s',a') + \sum_{s' \in S_{k(s)-1}} f(s') \rbr{ \sum_{a'\in A}\pi_1(a'|s')P(s|s',a') }
	\end{equation}
	so that we can apply \pref{lem:induction_state_reach_prob} (clearly we have $f(s_0) = 0$).
	To prove \pref{eq:key_ineq_sqrt_diff}, consider a fixed state $s\in S_k$ for some $k>0$. We rewrite $q_1(s)-q_2(s)$ as 
	\[
	\begin{split}
	q_1(s) - q_2(s) & = \sum_{s' \in S_{k-1}}\sbr{q_1(s')\sum_{a'\in A}\pi_1(a'|s')P(s|s',a') - q_2(s')\sum_{a'\in A}\pi_2(s',a')P(s|s',a')}.
	\end{split}
	\]
	
	For state $s'\in S_{k-1}$ satisfying $q_1(s') \leq q_2(s')$, we have 
	\[
	\begin{split}
	& q_1(s')\sum_{a'\in A}\pi_1(a'|s')P(s|s',a') - q_2(s')\sum_{a'\in A}\pi_2(a'|s')P(s|s',a') \\
	& \leq q_1(s')\sum_{a'\in A}\pi_1(a'|s')P(s|s',a') - q_1(s')\sum_{a'\in A}\pi_2(a'|s')P(s|s',a') \\
	& = q_1(s')\rbr{\sum_{a'\in A}\rbr{\pi_1(a'|s')-\pi_2(a'|s')}P(s|s',a')} \\
	& \leq q_1(s')\rbr{\sum_{a'\neq \pi_2(s')}\rbr{\pi_1(a'|s')}P(s|s',a')} \\ 
	& = \sum_{a\neq \pi_2(s')}q_1(s',a')P(s|s',a').
	\end{split}
	\]
	where the forth line follows from the fact $\pi_1(a'|s') - \pi_2(a'|s')=\pi_1(a'|s') - 1\leq 0$ when $a'= \pi_2(s')$. 
	
	For the other states with $q_1(s') \geq q_2(s')$, we have
	\[
	\begin{split}
	& q_1(s')\sum_{a'\in A}\pi_1(a'|s')P(s|s',a') - q_2(s')\sum_{a'\in A}\pi_2(a'|s')P(s|s',a') \\
	& = \rbr{q_1(s')-q_2(s')}\sum_{a'\in A}\pi_1(a'|s')P(s|s',a') + q_2(s')\sum_{a'\in A}\rbr{\pi_1(a'|s') - \pi_2(a'|s')}P(s'|s',a) \\
	& \leq \rbr{q_1(s')-q_2(s')}\sum_{a'\in A}\pi_1(a'|s')P(s|s',a') + q_2(s')\sum_{a'\neq \pi_2(s')}\pi_1(a'|s') P(s'|s',a)  \\
	& \leq \mathbb{I}_{s'}\rbr{q_1(s')-q_2(s')}\sum_{a'\in A}\pi_1(a'|s')P(s|s',a') + \sum_{a'\neq \pi_2(s')}q_1(s',a') P(s'|s',a) 
	\end{split}
	\]
	where the last line is due to the condition $q_1(s')\geq q_2(s')$. 
	
	Combining these two cases together yields \pref{eq:key_ineq_sqrt_diff}. Therefore, applying \pref{lem:induction_state_reach_prob} gives
	\[
	\begin{split}
	\mathbb{I}_{s}\rbr{q_1(s)-q_2(s)} & \leq \sum_{k=0}^{k(s)-1}\sum_{s'\in S_k} \sum_{a'} g(s',a') p^{\pi_1}(s|s',a')\\
	& = \sum_{k=0}^{k(s)-1}\sum_{s'\in S_k} \sum_{a' \neq \pi_2(s')} q_1(s',a') p^{\pi_1}(s|s',a'),
	\end{split}
	\]
	for all $s \neq s_0$.
	Taking square root and summation over all states, we have 
	\[
	\sum_{s\neq s_L}\sqrt{\mathbb{I}_s\rbr{q_1(s) - q_2(s)}} 
	\leq \sum_{s\neq s_L} \sqrt{ \sum_{k = 0}^{k(s)-1} \sum_{s' \in S_k}   \sum_{a'\neq \pi_2(s')}q_1(s',a')p^{\pi_1}(s|s',a')}.
	\]
 Notice that
	\begin{align*}
	 \sum_{s\neq s_L} \sum_{k = 0}^{k(s)-1} \sum_{s' \in S_k}   \sum_{a'\neq \pi_2(s')}q_1(s',a')p^{\pi_1}(s|s',a') 
	&\leq \sum_{s' \neq s_L} \sum_{a'\neq \pi_2(s')}q_1(s',a') \rbr{\sum_{s\neq s_L}p^{\pi_1}(s|s',a')} \\
	&\leq L\sum_{s \neq s_L} \sum_{a\neq \pi_2(s)}q_1(s,a), 
	\end{align*}
	where the last step uses \pref{lem:state_reach_probabilty}.
	Finally by \Holder's inequality, we arrive at 
	\[
	\begin{split}
	\sum_{s\neq s_L }\sqrt{q_1(s)} - \sqrt{q_2(s)} &\leq \sum_{s\neq s_L} \sqrt{ \sum_{k = 0}^{k(s)-1} \sum_{s' \in S_k}   \sum_{a'\neq \pi_2(s')}q_1(s',a')p^{\pi_1}(s|s',a')}   \\
	& \leq  \rbr{L\sum_{s \neq s_L} \sum_{a\neq \pi_2(s)}q_1(s,a)}^{\nicefrac{1}{2}} \rbr{\sum_{s\neq s_L} 1}^{\nicefrac{1}{2}}  \\
	& = \sqrt{|S|L} \sqrt{\sum_{s \neq s_L} \sum_{a\neq \pi_2(s)}q_1(s,a)},
	\end{split}	
	\]
	which concludes the proof. 
\end{proof}

\begin{lemma} \label{lem:sqrt_diff_2} For two deterministic policies $\pi_1$ and $\pi_2$, we have 
	\[
	\sum_{s\neq s_L} \sqrt{q_1(s)} - \sqrt{q_2(s)} 
	\leq \sqrt{|S|L}\sqrt{\sum_{s \neq s_L}\sum_{a\neq \pi_1(s)}q_2(s,a)} 
	\] 
	where we denote by $q_1 = q^{\pi_1}$ and $q_2=q^{\pi_2}$ the occupancy measures of $\pi_1$ and $\pi_2$ respectively. 
\end{lemma}
\begin{proof}
	The proof is similar to that of \pref{lem:sqrt_diff}. 
	First note that $\sqrt{q_1(s)} - \sqrt{q_2(s)}  \leq \sqrt{\mathbb{I}_s\rbr{q_1(s)-q_2(s)}} $ where
   $\mathbb{I}_s \triangleq \Ind{q_1(s) \geq q_2(s)}$. 
	Define $f(s) = \mathbb{I}_s\rbr{q_1(s) - q_2(s)}$ and $g(s,a) = \pi_1(a|s) \sum_{a'\neq \pi_1(s)}q_2(s,a')$.
	Our goal is to prove the following inequality
	\[
	f(s) \leq \sum_{s' \in S_{k(s)-1}}\sum_{a'\in A}g(s',a')P(s|s',a') + \sum_{s' \in S_{k(s)-1}}f(s')\rbr{\sum_{a'\in A}\pi_1(a'|s')P(s|s',a')}
	\] 
	for all $s\neq s_L$ so that we can apply \pref{lem:induction_state_reach_prob} again ($f(s_0) = 0$ holds clearly).
	To show this, consider a fixed state $s\in S_k$ for some $k>0$ and rewrite the term $q_1(s)-q_2(s)$ as
	\[
	\begin{split}
	q_1(s) - q_2(s) & = \sum_{s'\in S_{k-1}} \sum_{a'} \rbr{q_1(s',a')P(s|s',a') - q_2(s',a')P(s|s',a')} \\ 
	& = \sum_{s'\in S_{k-1}} \rbr{q_1(s')P(s|s',\pi_1(s')) - q_2(s')P(s|s',\pi_2(s'))}
	\end{split}
	\]
	since both $\pi_1$ and $\pi_2$ are deterministic polices.	
	Then for the states $s'\in S_{k-1}$ with $\pi_1(s') = \pi_2(s')$, we have
	\[
	q_1(s')P(s|s',\pi_1(s')) - q_2(s')P(s|s',\pi_2(s')) \leq \mathbb{I}_{s'}\rbr{q_1(s') - q_2(s')} P(s|s',\pi_1(s'))  
	\]
	For the other states with $\pi_1(s') \neq \pi_2(s')$, we have 
	\[
	\begin{split}
	& q_1(s')P(s|s',\pi_1(s')) - q_2(s')P(s|s',\pi_2(s')) \\ 
	& \leq \rbr{q_1(s') - q_2(s')} P(s|s',\pi_1(s'))  +  q_2(s') P(s|s',\pi_1(s')) - q_2(s') P(s|s',\pi_2(s')) \\ 
	& \leq \rbr{q_1(s') - q_2(s')} P(s|s',\pi_1(s'))  +  q_2(s') P(s|s',\pi_1(s')) - q_2(s',\pi_1(s')) P(s|s',\pi_1(s')) 
	\\ & = \rbr{q_1(s') - q_2(s')} P(s|s',\pi_1(s')) + \sbr{\sum_{a\neq \pi_1(s')}q_2(s',a)} P(s|s',\pi_1(s')) \\ 
	& = \rbr{q_1(s') - q_2(s')} \sum_{a'}\pi_1(a'|s') P(s|s',a') + \sum_{a'}\pi_1(a'|s')\sbr{\sum_{a\neq \pi_1(s')}q_2(s',a)}P(s|s',a') \\
	& \leq \mathbb{I}_{s'}\rbr{q_1(s') - q_2(s')} \sum_{a'}\pi_1(a'|s') P(s|s',a') + \sum_{a'}g(s',a')P(s|s',a'),
	\end{split}
	\]
	where the third line follows from the fact that $- q_2(s') P(s|s',\pi_2(s')) \leq - q_2(s',\pi_1(s')) P(s|s',\pi_1(s'))$, which holds because the right-hand side is simply zero when $\pi_1(s') \neq \pi_2(s')$.

	
	Combining these two cases proves the desired inequality for all states. Applying \pref{lem:induction_state_reach_prob}, we arrive at
	\[
	f(s) \leq \sum_{k=0}^{k(s)-1}\sum_{s'\in S_k}\sum_{a'\in A}g(s',a')p^{\pi_1}(s|s',a'). 
	\]
	By the same arguments used in the proof of \pref{lem:sqrt_diff}, we have shown
	\[
	\sum_{s\neq s_L} \sqrt{q_1(s)} - \sqrt{q_2(s)} \leq \sqrt{|S|L} \sqrt{ \sum_{s\neq s_L}\sum_{a\in A} g(s,a)}.
	\]
	Noticing that
	\[
	\sum_{s\neq s_L}\sum_{a\in A} g(s,a) = \sum_{s\neq s_L}\rbr{\sum_{a\in A}\pi_1(a|s)}\sbr{\sum_{a'\neq \pi_1(s)}q_2(s,a')} = \sum_{s\neq s_L}\sum_{a'\neq \pi_1(s)}q_2(s,a')
	\]
	finishes the proof. 
\end{proof}

Finally, we are ready to prove \pref{lem:key_lemma_bound_penalty}.
\begin{proof}[Proof of \pref{lem:key_lemma_bound_penalty}]	
Recall the calculation at the beginning of this section:
\[
\begin{split}
\phi_{H}(\opt) - \phi_{H}(q_t) & = \sum_{s \neq s_L}\sqrt{q_t(s)}h_s(\pi_t) - \sum_{s \neq s_L}\sqrt{\opt(s)}h_s(\optpi) \\
& = \sum_{s \neq s_L}\sqrt{q_t(s)}\rbr{h_s(\pi_t)-h_s(\optpi)} + \rbr{1 + \alpha(|A|-1)}\sum_{s \neq s_L}\rbr{\sqrt{q_t(s)} - \sqrt{\opt(s)}}. 
\end{split}
\]
Using \pref{lem:penalty_term_1} with $\pi_1 = \pi_t$ and $\pi_2 = \optpi$,
we bound the first term by $(1+\alpha) \sum_{s\neq s_L}\sum_{a \neq \pi(s)}\sqrt{q_t(s,a)}$ for any mapping $\pi: S \rightarrow A$.
Then we decompose the term $\sum_{s \neq s_L}\rbr{\sqrt{q_t(s)} - \sqrt{\opt(s)}}$ as
\[
\begin{split}
\sum_{s \neq s_L}\rbr{\sqrt{q_t(s)} - \sqrt{\opt(s)}} & = \sum_{s \neq s_L}\rbr{\sqrt{q_t(s)} - \sqrt{q^\pi(s)}} +  \sum_{s \neq s_L}\rbr{\sqrt{q^\pi(s)} - \sqrt{\opt(s)}} \\
& \leq \sqrt{|S|L}\rbr{\sqrt{\sum_{s \neq s_L}\sum_{a\neq \pi(s)}q_t(s,a)} + \sqrt{\sum_{s \neq s_L}\sum_{a\neq \pi(s)}\opt(s,a)}} \\
& \leq 2\sqrt{|S|L\sum_{s \neq s_L}\sum_{a\neq \pi(s)}q_t(s,a)+\opt(s,a)}
\end{split} 
\]
where the second line is by using \pref{lem:sqrt_diff} with $\pi_1 = \pi_t $ and $\pi_2 =\pi$,
and \pref{lem:sqrt_diff_2} with $\pi_1 = \pi$ and $\pi_2=\optpi $.
This provides one upper bound for $\sum_{s \neq s_L}\rbr{\sqrt{q_t(s)} - \sqrt{\opt(s)}}$.
On the other hand, it can also be trivially bounded as
\[
\sum_{s\neq s_L}\sqrt{q_t(s)}-\sqrt{\opt(s)} \leq \sum_{s\neq s_L}\sqrt{q_t(s)} \leq \sqrt{|S|L}
\]
where the second step uses Cauchy-Schwarz inequality and the fact $\sum_{s\neq s_L}q_t(s) = L$. 
Combining everything shows that $\phi_{H}(\opt) - \phi_{H}(q_t)$ is bounded by
\[
(1+\alpha)\sum_{s\neq s_L}\sum_{a\neq \pi(s)}\sqrt{q_t(s,a)} + (1+\alpha |A|)\sqrt{|S|L}\min\cbr{2\sqrt{\sum_{s\neq s_L}\sum_{a\neq \pi(s)}q_t(s,a)+\opt(s,a)}\;,1},
\]
finishing the proof.
\end{proof}

\section{Proof of \pref{lem:key_lemma_bound_stability}}
\label{app:appendix_proof_stability}


In this section, we provide the proof for \pref{lem:key_lemma_bound_stability}.
Throughout the section, we use the shorthand $H \triangleq \nabla^{2}\phi_{H}(q_t)$.
The key is clearly to analyze the Hessian inverse $H^{-1}$, which is done in \pref{app:Hessian_inverse}.
With a better understanding of the Hessian inverse, we then finish the proof in \pref{app:stability_proof}.

\subsection{Analyzing the Hessian inverse}
\label{app:Hessian_inverse}

To facilitate discussions, we first introduce some matrix notation.
We see the Hessian $H$ as a matrix in $\fR^{(|S||A|)\times(|S||A|)}$ in the natural way.\footnote{Technically, $S$ should be $S\setminus \{s_L\}$ instead.}
For subsets $D, E \subseteq S\times A$, we use $\fR^{D\times E}$ to represent the set of matrices (in $\fR^{|D|\times |E|}$ ) with the elements in $D$ indexing their rows and the elements in $E$ indexing their columns.
The notation $M((s,a), (s',a'))$ represents the entry of a matrix $M$ in the row indexed by $(s,a)$ and in the column indexed by $(s',a')$.
Let $U_k = \{(s,a): s \in S_k, a \in A\}$ for $k = 0, \ldots, L-1$, and $U_{j:k} = U_j \cup \cdots \cup U_k$.
We use similar notations for vectors, and define $\mathbf{0}_U \in \fR^U$ be the all-zero vector with elements in $U$ indexing its coordinates.

Define diagonal matrices $I_k, C_k, D_k \in \fR^{U_k \times U_k}$ as
\[
\begin{split}
I_k & = diag\cbr{1: (s,a) \in U_k}, \\
C_k & = diag\cbr{ c(s,a): (s,a) \in U_k }, \\
D_k & = diag\cbr{ d(s,a): (s,a) \in U_k }, 
\end{split}
\]
where \[
c(s,a) = \frac{\alpha}{4(q(s) - q(s,a))^{\nicefrac{3}{2}}} \quad\text{and}\quad 
 d(s,a) = \frac{1}{4q(s,a)^{\nicefrac{3}{2}}}.
\]
Also define transition matrices $P_k \in\fR^{U_{k-1} \times U_k}$ such that $P_k((s,a),(s',a')) = P(s'|s,a)$,
and $\widetilde{P}_k \in \fR^{U_{0:k-1} \times U_k}$ such that $P_k((s,a),(s',a')) = P(s'|s,a)$ if $s' \in S_{k(s)+1}$ and $P_k((s,a),(s',a')) = 0$ otherwise.

Our first step is to write $H$ in a recursive way with the help of a sequence of matrices:

\begin{lemma}\label{lem:def_M}
Define matrices $M_k \in \fR^{U_{0:k} \times U_{0:k}}$ for $k=0, \ldots, L-1$ recursively as
\begin{equation}\label{eq:def_M}
	M_k = \begin{pmatrix}
	M_{k-1} + \widetilde{P}_k C_k \widetilde{P}_k^\top & -\widetilde{P}_k C_k \\
	-C_k \widetilde{P}_k^\top & C_k + D_k 
	\end{pmatrix}
	= \begin{pmatrix}
	M_{k-1}  &0 \\
	0 &  D_k 
	\end{pmatrix} +
	\begin{pmatrix}
	\widetilde{P}_k \\
	-I_t
	\end{pmatrix}
	C_k 
	\begin{pmatrix}
	\widetilde{P}_k^\top &
	-I_t
	\end{pmatrix}
\end{equation}	
for $k = 1, \ldots, L-1$, and $M_0 = C_0 + D_0$.
Then we have $H = M_{L-1}$.
\end{lemma}

\begin{proof}
The proof is by a direct verification based on the calculation of the Hessian done in \pref{lem:hessian_derivaties}.
The claim is that $M_k$ consists of all the second-order derivatives with respect to $(s,a), (s', a') \in U_{0:k}$, but without the terms involving $P(s''|s,a)$ or $P(s''|s',a')$ for $s'' \in S_{k+1}$.
To see this, note that this is clearly true for $M_0$ based on \pref{lem:hessian_derivaties}.
Suppose this is true for $M_{k-1}$ and consider $M_k$.

We first show that the block $\widetilde{P}_k C_k \widetilde{P}_k^\top$ corresponds to \pref{eq:derivative_2} plus the last term of \pref{eq:derivative_3}, that is 
\begin{align*}
& \rbr{\widetilde{P}_k C_k \widetilde{P}_k^\top}((s,a),(s',a')) \\ & = 
\sum_{(s_1,a_1)\in U_k} \sum_{(s_2,a_2)\in U_k} \widetilde{P}_k((s,a);(s_1,a_1)) \cdot C_k((s_1,a_1);(s_2,a_2)) \cdot \widetilde{P}_k((s',a');(s_2,a_2))
\\
& = \sum_{(s_1,a_1)\in U_k} \widetilde{P}_k((s,a),(s_1,a_1)) \cdot C_k((s_1,a_1),(s_1,a_1)) \cdot \widetilde{P}_k((s',a'),(s_1,a_1)) \\ & \quad + \sum_{(s_1,a_1)\in U_k} \sum_{(s_2,a_2)\in U_k, (s_1,a_1)\neq (s_2,a_2)}  \widetilde{P}_k((s,a),(s_1,a_1)) \cdot C_k((s_1,a_1),(s_2,a_2)) \cdot \widetilde{P}_k((s',a'),(s_1,a_1)) \\
& = \sum_{(s_1,a_1)\in U_k} P(s_1|s,a) \cdot \frac{\alpha}{4(q(s_1) - q(s_1,a_1))^{\nicefrac{3}{2}}} \cdot P(s_1|s',a')  \\
& = \sum_{s'' \in S_k} \sum_{a'' \in A} \frac{\alpha P(s''|s,a)P(s''|s',a')}{4(q(s'') - q(s'',a''))^{\nicefrac{3}{2}}}
\end{align*}
where the third step uses the fact that $C_k((s_1,a_1),(s_2,a_2))=0$ when $(s_1,a_1)\neq(s_2,a_2)$.

Then we verify that the blocks  $ -\widetilde{P}_k C_k$ and $-C_k \widetilde{P}_k^\top$ correspond to \pref{eq:derivative_1}. Direct calculation shows that, for $(s,a) \in U_{k}$ and $(s',a') \in U_{k-1}$, 
\begin{align*}
\rbr{ -\widetilde{P}_k C_k}((s',a'),(s,a)) & = -\sum_{(s'',a'')\in U_k} \widetilde{P}_k( (s',a'),(s'',a'')) C_k((s'',a''),(s,a)) \\ & = - \frac{\alpha}{4(q(s) - q(s,a))^{\nicefrac{3}{2}}} \cdot P(s|s',a') .
\end{align*}

Finally, the block $ C_k + D_k $ corresponds to the first two terms of \pref{eq:derivative_3}.
This finishes the proof.
\end{proof}

To study $H^{-1} = M_{L-1}^{-1}$, we next write $M_k^{-1}$ in terms of $M_{k-1}^{-1}$.

\begin{lemma} \label{lem:M_inverse}
The inverse of $M_k$ (defined in \pref{eq:def_M}) is
\begin{align}
M_k^{-1} &= \begin{pmatrix}M_{k-1}^{-1} & 0\\ 0 & D_k^{-1} \end{pmatrix} - \begin{pmatrix}M_{k-1}^{-1}\widetilde{P}_k \\ -D_k^{-1} \end{pmatrix}
\rbr{ C_k^{-1} + D_k^{-1}+\widetilde{P}_k^\top M^{-1}_{t-1}\widetilde{P}_k}^{-1}  
\begin{pmatrix}\widetilde{P}_k^\top M_{k-1}^{-1} & -D_k^{-1}\end{pmatrix}
\notag \\& =\begin{pmatrix}
M_{k-1}^{-1} - M_{k-1}^{-1} \widetilde{P}_k W_k \widetilde{P}_k^\top  M_{k-1}^{-1} & M_{k-1}^{-1} \widetilde{P}_k W_k D_k^{-1}\\
D_k^{-1} W_k \widetilde{P}_k^\top M_{k-1}^{-1}  & D_k^{-1} - D_k^{-1} W_k D_k^{-1} 
\end{pmatrix}   \label{eq:M_k_inv_decomp}
\end{align}
where $W_k = ( C_k^{-1} + D_k^{-1} + \widetilde{P}_k^\top M_{k-1}^{-1} \widetilde{P}_k )^{-1}$. 
\label{lem:M_k_inverse_decomp}
\end{lemma}

\begin{proof}
The first equality is by the Woodbury matrix identity 
\[
(A+UBV)^{-1} = A^{-1} - A^{-1}U(B^{-1}+VA^{-1}U)^{-1}VA^{-1}
\]
and plugging in the definition of $M_k$ from \pref{eq:def_M} with 
\[
A = \begin{pmatrix}
	M_{k-1}  &0 \\
	0 &  D_k 
	\end{pmatrix},
	\quad
U = V^\top = \begin{pmatrix}
	\widetilde{P}_k \\
	-I_t
	\end{pmatrix},
\quad\text{and}\quad B =	C_k .
\]	
The second equality is by direct calculation.
\end{proof}

The bottom right block of $M_k^{-1}$ plays a key role in the analysis, and is denoted by
\begin{equation}\label{eq:def_N}
N_k = D_k^{-1} - D_k^{-1} W_k D_k^{-1}
\end{equation}
for $k=1, \ldots, L-1$, and $N_0 = M_0^{-1}$.
The next lemma shows that we can focus on $N_k$ when analyzing specific quadratic forms of $H^{-1}$.

\begin{lemma}\label{lem:H_and_N}
For any vector $w_k\in \fR^{U_k}$, we have 
\[
	\begin{pmatrix}
	\mathbf{0}_{U_{0:k-1}}^\top,w_k^\top, \mathbf{0}_{U_{k+1:L-1}}^\top
	\end{pmatrix} H^{-1} \begin{pmatrix}
	\mathbf{0}_{U_{0:k-1}} \\ w_k \\\mathbf{0}_{U_{k+1:L-1}}
	\end{pmatrix} \leq w_k^\top N_k w_k,
\]
	where $N_k$ is defined in \pref{eq:def_N}. 
\end{lemma}
\begin{proof} 
Based on \pref{eq:def_M} and the fact that $C_k$ is positive definite, we have
\[
M_k \succeq \begin{pmatrix}
	M_{k-1}  &0 \\
	0 &  D_k 
	\end{pmatrix}.
\]
Repeatedly using this fact, we can show
\[
H = M_{L-1} \succeq 
\begin{pmatrix}
	M_{k}  & 0  & \cdots & 0 \\
	0 &  D_{k+1} & \cdots & 0 \\
	\vdots & \vdots & \ddots & \vdots \\
	0 & 0 & \cdots & D_{L-1} \\
\end{pmatrix},
\]
and thus
\[
H^{-1} = M_{L-1}^{-1} \preceq 
\begin{pmatrix}
	M_{k}^{-1}  & 0  & \cdots & 0 \\
	0 &  D_{k+1}^{-1} & \cdots & 0 \\
	\vdots & \vdots & \ddots & \vdots \\
	0 & 0 & \cdots & D_{L-1}^{-1} \\
\end{pmatrix}.
\]

Note that for the last matrix above, the block with rows and columns indexed by elements in $U_k$ is exactly $N_k$, based on \pref{lem:M_inverse}. 
Thus, taking the quadratic form on both sides with respect to the vector $	\begin{pmatrix}
	\mathbf{0}_{U_{0:k-1}}^\top,w_k^\top, \mathbf{0}_{U_{k+1:L-1}}^\top
	\end{pmatrix} $ finishes the proof.
\end{proof}

Finally, we point out some important properties of $N_k$.
\begin{lemma}\label{lem:N_properties}
The matrix $N_k$ (defined in \pref{eq:def_N}) is positive definite and satisfies
\[
N_k \preceq D_k^{-1}  \quad\text{and}\quad N_k \preceq C_k^{-1}+P_k^\top N_{k-1} P_k.
\]
\end{lemma}
\begin{proof}
The fact that $N_k$ is positive definite is directly implied by \pref{lem:H_and_N}.
To prove the rest of the statement, we first apply Woodbury matrix identity to write $W_k$ as
\begin{align*}
W_k =  ( D_k^{-1} + C_k^{-1} + \widetilde{P}_k^\top M_{k-1}^{-1} \widetilde{P}_k )^{-1}
 &= D_k - D_k ( D_k + (C_k^{-1} +\widetilde{P}_k^\top M_{k-1}^{-1} \widetilde{P}_k )^{-1})^{-1} D_k \\
 &= D_k - D_k ( D_k + (C_k^{-1} + P_k^\top N_{k-1}^{-1} P_k )^{-1})^{-1} D_k. 
\end{align*}
Plugging this back into the definition of $N_k$ gives:
\[
N_k = D_k^{-1} - D_k^{-1} W_k D_k^{-1} =  ( D_k + (C_k^{-1} + P_k^\top N_{k-1}^{-1} P_k )^{-1})^{-1},
\]
which shows $N_k \preceq D_k^{-1}$ and $N_k \preceq C_k^{-1}+P_k^\top N_{k-1} P_k$.
\end{proof}

\subsection{Bounding the stability term}
\label{app:stability_proof}

With the tools from the previous section, we are now ready to bound the stability term.
We will use the following lemma to relate the quadratic form of $H^{-1}$ to only its diagonal entries.

\begin{lemma}\label{lem:quadratic_form_diagonal}
If $M \in \fR^{d\times d}$ is a positive semi-definite matrix, then for any $w\in \fR^d$ with non-negative coordinates, we have
\[
w^\top M w \leq \rbr{\sum_{j=1}^d w(j)}\sum_{i=1}^d M(i,i)w(i). 
\]
\end{lemma}
\begin{proof}
Since $M$ is positive semi-definite, we have for any $i, j$, $(e_i - e_j)^\top M (e_i - e_j) \geq 0$, which implies 
\[
M(i,j) = M(j, i) \leq \frac{M(i,i) + M(j,j)}{2}.
\]
Therefore, 
\[
w^\top M w = \sum_{i,j} w(i) M(i,j) w(j) \leq
\frac{1}{2}\sum_{i,j} w(i) (M(i,i) + M(j,j)) w(j) = \rbr{\sum_{j=1}^d w(j)}\sum_{i=1}^d M(i,i)w(i),
\]
where we use the nonnegativity of $w$ in the second step as well.
\end{proof}

We now bound $\|\hatl_t\|_{H^{-1}}$ in terms of the diagonal entries of $N_k$.
\begin{lemma}\label{lem:stability_first_step}
\pref{alg:main_alg} guarantees
\[
\E\sbr{\norm{\hatl_t}_{H^{-1}}} 
\leq L\E\sbr{\sum_{k=0}^{L-1}\sum_{(s,a)\in U_k}{\frac{N_k((s,a), (s,a))}{q_t(s,a)}}}.
\]
\end{lemma}

\begin{proof}
Recall the definition of $\hatl_t$: $\hatl_t(s,a) = \frac{\ell_t(s,a)}{q_t(s,a)}\Ind{s, a}$ where we use the shorthand $\Ind{s, a} = \Ind{s_{k(s)}=s,a_{k(s)}= a}$. Therefore, we have
	\begin{equation}
	\begin{split}
	\E\sbr{\norm{\hatl_t}_{H^{-1}}} & = \E\sbr{\sum_{s,a}\sum_{s',a'}\frac{H^{-1}\rbr{(s,a), (s',a')}}{q_t(s,a)q_t(s',a')}\Ind{s,a}\ell_t(s,a)\Ind{s',a'}\ell_t(s',a')}\\
	& \leq \E\sbr{\rbr{\sum_{s',a'}\Ind{s',a'}\ell_t(s',a')}\sum_{s,a}\rbr{\frac{H^{-1}\rbr{(s,a),(s,a)}}{q_t(s,a)^2}}\Ind{s,a}\ell_t(s,a)},
  \end{split}
	\end{equation}
where in the last step we use \pref{lem:quadratic_form_diagonal} with $M$ being a matrix in the same shape of $H$ and with entry $M((s,a),(s',a')) = \frac{H^{-1}((s,a),(s',a'))}{q_t(s,a) q_t(s',a')}$ (which is clearly positive definite).
Using the fact $\sum_{s',a'}\Ind{s',a'}\ell(s',a') \leq L$, $H^{-1}\rbr{(s,a),(s,a)} \geq 0$, and $\ell_t(s,a) \in [0,1]$, we continue with
\begin{align*}
\E\sbr{\norm{\hatl_t}_{H^{-1}}} 
&\leq L\E\sbr{\sum_{s,a}\rbr{\frac{H^{-1}\rbr{(s,a),(s,a)}}{q_t(s,a)^2}}\Ind{s,a}} \\
&= L\E\sbr{\sum_{s,a}{\frac{H^{-1}\rbr{(s,a),(s,a)}}{q_t(s,a)}}} \\
&\leq L\E\sbr{\sum_{k=0}^{L-1}\sum_{(s,a)\in U_k}{\frac{N_k((s,a), (s,a))}{q_t(s,a)}}}, 
\end{align*}
where in the last step we use \pref{lem:H_and_N}.
\end{proof}

Next, we continue to bound the term involving $N_k$ using the following lemma.
\begin{lemma}\label{lem:stability_second_step}
\pref{alg:main_alg} guarantees
\[
\sum_{k=0}^{L-1}\sum_{(s,a)\in U_k}{\frac{N_k((s,a), (s,a))}{q_t(s,a)}} \leq
8eL\rbr{\sqrt{L}+ \frac{1}{\alpha L}} \sum_{s\neq s_L}\sum_{a\neq \pi(s)} \sqrt{q_t(s,a)}
\]
for any mapping $\pi$ from $S$ to $A$.
\end{lemma}

\begin{proof}
For notational convenience, define $R(s,a) = N_{k(s)}((s,a), (s,a))$.
We first prove that for any $k=1, \ldots, L-1$, 
\begin{equation}\label{eq:R}
	\sum_{(s,a)\in U_k}\frac{R(s,a)}{q_t(s,a)} \leq 8\rbr{\sqrt{L} + \frac{1}{\alpha L}} \sum_{s \in S_k}\sum_{a\neq \pi(s)} \sqrt{q_t(s,a)} + \rbr{ 1 + \frac{1}{L}} \sum_{(s', a')\in U_{k-1}}\frac{R(s',a')}{q_t(s',a')},
\end{equation}
and for $k = 0$,	
\begin{equation}\label{eq:R0}
\sum_{a \in A}\frac{R(s_0,a)}{q_t(s_0,a)} \leq 8\rbr{ \sqrt{L}+ \frac{1}{\alpha L}} \sum_{a\neq \pi(s)}\sqrt{q(s_0,a)},
\end{equation}
where $\pi$ is any mapping from $S$ to $A$.
Indeed, repeatedly applying \eqref{eq:R} and using the fact $(1+\nicefrac{1}{L})^L \leq e$ show
\[
\sum_{(s,a)\in U_k}\frac{R(s,a)}{q_t(s,a)} \leq 8e\rbr{\sqrt{L}+ \frac{1}{\alpha L}}\sum_{l=0}^k\sum_{s\in S_l}\sum_{a\neq \pi(s)} \sqrt{q_t(s,a)},
\]
and summing over $k$ finishes the proof.

To prove \pref{eq:R0}, note that by definition, when $s=s_0$ we have
\[
R(s,a) = \frac{4}{q_t(s,a)^{-\nicefrac{3}{2}} + \alpha(q_t(s)-q_t(s,a))^{-\nicefrac{3}{2}}} \leq 4\min\cbr{q_t(s,a)^{\nicefrac{3}{2}}, \frac{(q_t(s)-q_t(s,a))^{\nicefrac{3}{2}} }{\alpha}}.
\]
Now consider two cases, if $\frac{q_t(s)-q_t(s,\pi(s))}{q_t(s,\pi(s))} \leq \frac{1}{L}$, then
\[
\frac{R(s,\pi(s))}{q_t(s,\pi(s))} \leq \frac{4}{\alpha L}\sqrt{\sum_{a\neq  \pi(s)}q_t(s,a)}
\leq \frac{4}{\alpha L}\sum_{a\neq  \pi(s)} \sqrt{q_t(s,a)}.
\]
On the other hand, if $\frac{q_t(s)-q_t(s,\pi(s))}{q_t(s,\pi(s))} > \frac{1}{L}$, then $q_t(s,\pi(s)) \leq L(q_t(s)-q_t(s,\pi(s)))$ and 
\[
\frac{R(s,\pi(s))}{q_t(s,\pi(s))} \leq 4\sqrt{q_t(s,\pi(s))}
\leq 4\sqrt{L}\sqrt{q_t(s)-q_t(s,\pi(s))} \leq 4\sqrt{L}\sum_{a\neq  \pi(s)}\sqrt{q_t(s,a)}.
\]
Combining the two cases and also the fact $\sum_{a\neq\pi(s)}\frac{R(s,a)}{q_t(s,a)} \leq 4\sum_{a\neq  \pi(s)}\sqrt{q_t(s,a)} $ proves \pref{eq:R0}.
	
   It remains to prove \pref{eq:R}.
   First, using the fact $N_k \preceq D_k^{-1}$ from \pref{lem:N_properties}, we again have 
   \begin{equation}\label{eq:R_bound}
   R(s,a) \leq 4q_t(s,a)^{\nicefrac{3}{2}}.
   \end{equation}
   At the same time, using another fact $N_k \preceq C_k^{-1} + P_k^\top N_{k-1} P_k$ from \pref{lem:N_properties} and shorthand $R(s,a,s',a') \triangleq N_{k(s)}((s,a), (s',a'))$, we have
	\begin{equation*}
	\begin{split}
	R(s,a) &\leq  \frac{4(q(s)-q(s,a))^{\nicefrac{3}{2}} }{\alpha} + \sum_{(s_1,a_1), (s_2,a_2)\in U_{k-1}}P(s|s_1,a_1)P(s|s_2,a_2)R(s_1,a_1,s_2,a_2) \\ 
	&=  \frac{4(q(s)-q(s,a))^{\nicefrac{3}{2}} }{\alpha} \\ &\quad  + \sum_{(s_1,a_1), (s_2,a_2)\in U_{k-1}}P(s|s_1,a_1)q_t(s_1,a_1)P(s|s_2,a_2)q_t(s_2,a_2)\frac{R(s_1,a_1,s_2,a_2)}{q_t(s_1,a_1)q_t(s_2,a_2)} \\ 
	&\leq  \frac{4(q(s)-q(s,a))^{\nicefrac{3}{2}} }{\alpha} \\ &\quad +\rbr{\sum_{(s_2,a_2)\in U_{k-1}}P(s|s_2,a_2)q_t(s_2,a_2)} \sum_{(s_1,a_1)\in U_{k-1}}P(s|s_1,a_1)q_t(s_1,a_1)\frac{R(s_1,a_1)}{q(s_1,a_1)^2} \\ 
	&=  \frac{4(q(s)-q(s,a))^{\nicefrac{3}{2}} }{\alpha} + q_t(s)\sum_{(s_1,a_1)\in U_{k-1}}P(s|s_1,a_1)\frac{R(s_1,a_1)}{q_t(s_1,a_1)},
	\end{split}
	\end{equation*}
	where the second inequality is by applying \pref{lem:quadratic_form_diagonal} again, with $M \in\fR^{U_{k-1}\times  U_{k-1}} $ such that $M((s_1,a_1), (s_2,a_2)) = \frac{R(s_1,a_1,s_2,a_2)}{q_t(s_1,a_1)q_t(s_2,a_2)}$ (which is positive definite by \pref{lem:N_properties}).
	Again, we fix $s$ and consider two cases.
	First, if $\frac{q_t(s)-q_t(s,\pi(s))}{q_t(s,\pi(s))} \leq \frac{1}{L}$, then 
	\begin{equation*}
	\label{eq:bound_optimal_case_2}
	\begin{split}
	\frac{R(s,\pi(s))}{q_t(s,\pi(s))} & \leq \frac{4(q_t(s)-q_t(s,\pi(s)))^{\nicefrac{3}{2}}}{\alpha q_t(s,\pi(s))} + \frac{q_t(s)}{q_t(s,\pi(s))}\sum_{(s_1,a_1)\in U_{k-1}}P(s|s_1,a_1)\frac{R(s_1,a_1)}{q_t(s_1,a_1)} \\
	& \leq \frac{4(q_t(s)-q_t(s,\pi(s)))}{\alpha   q_t(s,\pi(s))}  \sum_{a \neq \pi(s)}\sqrt{q_t(s,a)} + \rbr{1 + \frac{1}{L}} \sum_{(s_1,a_1)\in U_{k-1}}P(s|s_1,a_1)\frac{R(s_1,a_1)}{q_t(s_1,a_1)} \\ 
	& \leq \frac{4}{\alpha L}  \sum_{a \neq \pi(s)}\sqrt{q_t(s,a)} + \rbr{1 + \frac{1}{L}} \sum_{(s_1,a_1)\in U_{k-1}}P(s|s_1,a_1)\frac{R(s_1,a_1)}{q_t(s_1,a_1)}.
	\end{split}
	\end{equation*}
   
   On the other hand, if $\frac{q_t(s)-q_t(s,\pi(s))}{q_t(s,\pi(s))} > \frac{1}{L}$ and thus $q_t(s,\pi(s)) \leq L(q_t(s)-q_t(s,\pi(s)))$, then using \pref{eq:R_bound} we have
	\begin{equation*}
	\label{eq:bound_optimal_case_1}
	\frac{R(s,\pi(s))}{q(s,\pi(s))} \leq 4\sqrt{q(s,\pi(s))} \leq 4\sqrt{L \sum_{a \neq \pi(s)}q_t(s,a)}\leq 4\sqrt{L}\sum_{a \neq \pi(s)}\sqrt{q_t(s,a)}.
	\end{equation*}
	Combining the two cases and also $\sum_{a\neq \pi(s)}\frac{R(s,a)}{q_t(s,a)} \leq 4\sum_{a \neq \pi(s)}\sqrt{q_t(s,a)}$ (using  \pref{eq:R_bound}  again) leads to
	\[
	\sum_{a}\frac{R(s,a)}{q_t(s,a)}  \leq 8\rbr{\sqrt{L} +\frac{1}{\alpha L} } \sum_{a \neq \pi(s)}\sqrt{q_t(s,a)}+ \rbr{1 + \frac{1}{L}} \sum_{(s_1,a_1)\in U_{k-1}}P(s|s_1,a_1)\frac{R(s_1,a_1)}{q_t(s_1,a_1)}.
	\]

	Finally, we sum over all $s \in S_k$ and obtain
	\[
	\begin{split}
	&\sum_{(s,a) \in U_k} \frac{R(s,a)}{q_t(s,a)} \\
	& \leq 8\rbr{\sqrt{L} +\frac{1}{\alpha L} } \sum_{s \in S_k}\sum_{a\neq \pi(s)} \sqrt{q_t(s,a)} + \rbr{ 1 + \frac{1}{L}} \sum_{s \in S_k}\sum_{(s',a')\in U_{k-1}}P(s|s',a')\frac{R(s',a')}{q_t(s',a')} \\
	& =  8\rbr{\sqrt{L} +\frac{1}{\alpha L}} \sum_{s \in S_k}\sum_{a\neq \pi(s)} \sqrt{q_t(s,a)} + \rbr{ 1 + \frac{1}{L}} \sum_{(s',a')\in U_{k-1}}\frac{R(s',a')}{q_t(s',a')}.
	\end{split}
	\]
	This proves \pref{eq:R} and thus finishes the proof.
\end{proof}

We are now ready to finish the proof for \pref{lem:key_lemma_bound_stability}.
\begin{proof}[Proof of \pref{lem:key_lemma_bound_stability}]
Combining \pref{lem:stability_first_step} and \pref{lem:stability_second_step}, we prove 
\[
\E\sbr{\|\hatl_t\|^2_{\nabla^{-2}\phi_{H}(q_t)}} \leq \E\sbr{8eL^2\rbr{\sqrt{L} + \frac{1}{\alpha \cdot L} } \sum_{s\neq s_L}\sum_{a\neq \pi(a)}\sqrt{q_t(s,a)}}.
\]
It thus remains to prove the other bound
\[
\E\sbr{\|\hatl_t\|^2_{\nabla^{-2}\phi_{H}(q_t)}} \leq 4\sqrt{L|S||A|}.
\]
This is simply by considering only the regular ${\nicefrac{1}{2}}$-Tsallis entropy part of the regularizer: $\phi_{D}(q) = -\sum_{s,a}\sqrt{q(s,a)}$.
Specifically, by \pref{lem:hessian_derivaties} we have $\nabla^{2}\phi_{H}(q_t) \succeq \nabla^{2}\phi_{D}(q_t)$, and thus
	\[
	\begin{split}
	\E\sbr{\norm{\hatl_t}_{\nabla^{-2}\phi_{H}(q_t)}^2} & \leq \E\sbr{\norm{\hatl_t}_{\nabla^{-2}\phi_{D}(q_t)}^2} \\
	& = \E\sbr{\sum_{s,a} \frac{4q_t(s,a)^{\nicefrac{3}{2}}}{q_t(s,a)^2}\Ind{s,a}\ell_t(s,a)^2} \\ 
	& \leq 4\E\sbr{\sum_{s, a} \sqrt{q_t(s,a)} } \\ 
	& \leq 4 \sqrt{L|S||A|},
	\end{split}
	\]
	where the last step uses the Cauchy-Schwarz inequality. 
\end{proof}

\end{document}